\documentclass[twoside]{article}

\usepackage[accepted]{aistats2024}

\usepackage{
  amsfonts,
  amsmath,
  amsthm,
  booktabs,
  cancel,
  csvsimple,
  enumitem,
  graphicx,
  url,
  wrapfig,
  xcolor,
  xfrac,
}
\usepackage[round,sort&compress]{natbib}
\usepackage[font=small,labelfont=bf,labelsep=quad,justification=justified,singlelinecheck=false]{caption}
\usepackage[font=small,labelfont=bf,labelsep=quad,singlelinecheck=true]{subcaption}
\usepackage[hyperfootnotes=false]{hyperref}
\usepackage[capitalise,nameinlink,noabbrev]{cleveref}

\newcommand{\email}[1]{\href{mailto:#1}{\nolinkurl{#1}}}

\newcommand{\shorturl}[1]{\href{https://#1}{\nolinkurl{#1}}}

\hypersetup{
  colorlinks,
  linkcolor={blue!50!black},
  citecolor={blue!50!black},
  urlcolor={blue!50!black},
}

\creflabelformat{equation}{#2\textup{#1}#3}

\setlist[enumerate]{itemsep=2pt,topsep=0pt}
\setlist[itemize]{itemsep=2pt,topsep=0pt}

\newlength{\figureheight}
\usepackage{etoolbox}
\patchcmd{\toprule}{\heavyrulewidth}{0.5pt}{}{}
\patchcmd{\bottomrule}{\heavyrulewidth}{0.5pt}{}{}

\usepackage{array}
\newcolumntype{t}{>{\ttfamily}r}

\newcommand{\entropy}[1]{\mathrm{H} \! \left[ #1 \right]}
\newcommand{\mutualinfo}[1]{\mathrm{I} \! \left( #1 \right)}
\newcommand{\expectation}[2]{\mathbb{E}_{#1} \! \left[ #2 \right]}

\newcommand{\dpool}[0]{\mathcal{D}_\mathrm{pool}}

\newcommand{\dtrain}[0]{\mathcal{D}_\mathrm{train}}

\newtheoremstyle{proposition}  
    {10pt}  
    {0pt}  
    {\itshape}  
    {}  
    {\bfseries}  
    {\ }  
    { }  
    {}  
\theoremstyle{proposition}
\newtheorem{proposition}{Proposition}

\makeatletter
\renewenvironment{proof}[1][\proofname]{\par\pushQED{\qed}\normalfont\topsep6\p@\@plus6\p@\relax\trivlist\item[\hskip\labelsep\bfseries#1\ ]\ignorespaces}{\popQED\endtrivlist\@endpefalse}
\makeatother

\begin{document}

\raggedbottom

\twocolumn[
  \aistatstitle{Making Better Use of Unlabelled Data in Bayesian Active Learning}
  \aistatsauthor{Freddie Bickford Smith \And Adam Foster \And  Tom Rainforth}
  \aistatsaddress{University of Oxford \And Microsoft Research \And University of Oxford}
]

\begin{abstract}
  Fully supervised models are predominant in Bayesian active learning.
We argue that their neglect of the information present in unlabelled data harms not just predictive performance but also decisions about what data to acquire.
Our proposed solution is a simple framework for semi-supervised Bayesian active learning.
We find it produces better-performing models than either conventional Bayesian active learning or semi-supervised learning with randomly acquired data.
It is also easier to scale up than the conventional approach.
As well as supporting a shift towards semi-supervised models, our findings highlight the importance of studying models and acquisition methods in conjunction.
\end{abstract}

\section{Introduction}
\label{sec:introduction}

Bayesian active learning \citep{gal2017deep,houlsby2011bayesian,kirsch2023advancing,mackay1992evidence,mackay1992information} involves seeking the most informative labels for training a given model.
The model is the basis not only for prediction but also for deciding what data to acquire, with the latter requiring a notion of how the model's uncertainty will change after updating on new data.
Choosing the right model is therefore critical to success.

We make a case against the fully supervised models conventionally used in Bayesian active learning \citep{atighehchian2020bayesian,beluch2018power,bickfordsmith2023prediction,chitta2018large,gal2017deep,houlsby2011bayesian,jeon2020thompsonbald,kirsch2019batchbald,kirsch2023black,kirsch2023speeding,kirsch2023advancing,kirsch2023stochastic,murray2021depth,pinsler2019bayesian,pop2018deep,tran2019bayesian}.
Crucially these models do not account for all the information at hand: they neglect the rich information often conveyed by unlabelled data, which is usually assumed to be cheaply available \citep{lewis1994sequential,settles2012active}.
This neglected information can be critical to both making predictions and estimating the reducible uncertainty present in the model.
As a result the conventional approach can fail to appropriately target the novel information that new labels can provide.

This problem is exacerbated by the common practice of using big models and assuming all the model parameters will be updated after each data-acquisition step.
The consequences of this include significant redundancy in parameter uncertainty and inconsistencies in reducible-uncertainty estimation.
While the former might be addressed by focusing on predictions instead of parameters \citep{bickfordsmith2023prediction}, the latter remains a thorny problem.
The (implicit) priors used in practice can poorly reflect our beliefs about the reducible uncertainty present in the model, and it is usually not possible to accurately approximate the Bayesian updates that formal notions of reducible uncertainty assume \citep{sharma2023bayesian}.
Difficulties like these can pose a barrier to capturing how a model's uncertainty will change as new data is acquired.

Targeting updates in all the model parameters can additionally be computationally demanding.
Even when resorting to crude inference approximations or heuristics, updating all the model parameters after each data-acquisition step tends to be costly, and estimating the model's reducible uncertainty can also be challenging.
These high costs limit the applicability of conventional Bayesian active learning and often necessitate additional approximations \citep {kirsch2019batchbald}.

We argue that all these problems can be addressed by rethinking the conventional model setup with an emphasis on the specific needs of Bayesian active learning.
In particular we suggest the field should shift towards using semi-supervised models that can capitalise on unlabelled data while also producing useful estimates of their reducible uncertainty.

\begin{figure*}[t]
    \centering
    \includegraphics[height=\figureheight]{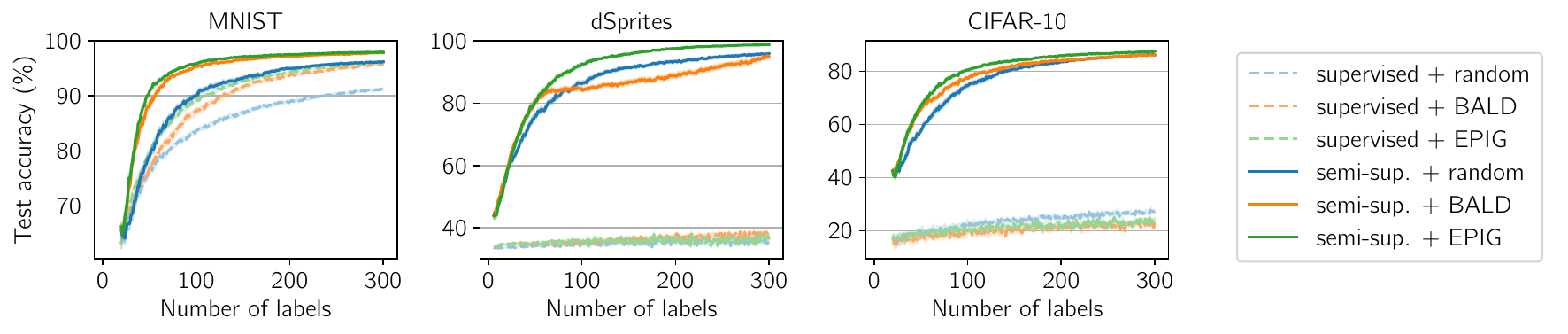}
    \caption{
        Semi-supervised models produce much better predictions than fully supervised ones.
        At the same time, using the right acquisition method is critical for effective active learning.
        BALD, which targets direct reductions in parameter uncertainty, does not consistently outperform random acquisition here; EPIG, a prediction-oriented alternative, does.
    }
    \label{fig:pretraining_helps}
\end{figure*}

To this end we propose a simple framework: pretrain a deterministic encoder on unlabelled data, then train a stochastic prediction head using Bayesian active learning.
The encoder (which we choose to keep fixed) allows us to incorporate relevant information from the unlabelled data into the model, and contains much of the model's processing power.
The prediction head is then a more lightweight component that captures label information and that, relative to the conventional model setup, supports more faithful estimation of the model's reducible uncertainty by allowing cheaper and more accurate updates as new data is acquired.

Our approach can produce notably better predictive performance than the conventional alternative of using fully supervised models (see \Cref{fig:pretraining_helps}).
It also significantly reduces the computational cost of training and acquisition.
As a result it opens up new practical use cases, including scaling up to high-dimensional input spaces and large pools of unlabelled data.

Given its simplicity and benefits, why is this approach not already the default in Bayesian active learning?
One explanation is an implicit assumption in past work that acquisition methods can be studied in isolation from considerations about the model.
We suggest this assumption is misguided: accounting for unlabelled data has a fundamental effect on what labels should be acquired.
Another explanation is a degree of ambiguity in the literature about the benefit of active learning, Bayesian or otherwise, when using semi-supervised models (see \Cref{sec:related_work}).
We show that using the right acquisition method is key to the success of semi-supervised Bayesian active learning.
In particular we find that EPIG \citep{bickfordsmith2023prediction}, a prediction-oriented method, produces more consistent gains over random acquisition than BALD \citep{houlsby2011bayesian}, a parameter-oriented method, does.

Our contribution is thus threefold.
First, we explain why the fully supervised models predominantly used in Bayesian active learning can undermine data acquisition.
Second, we set out a framework for using semi-supervised models instead, and show empirically that this outperforms training the same models on randomly acquired labels.
Third, we provide some clarity on how the field can move forward: it should embrace semi-supervised models and carefully study the interplay between models and acquisition methods.
\section{Background}
\label{sec:background}

We focus on learning a probabilistic predictive model, $p_{\phi}(y|x)$, where $x$ is an input and $y$ is a label.
We assume there are some underlying stochastic model parameters, $\theta \sim p_\phi(\theta)$, defined such that
\begin{align*}
    p_{\phi}(y|x)             & = \expectation{p_{\phi}(\theta)}{p_{\phi}(y|x,\theta)}                             \\
    p_{\phi}(y_1,y_2|x_1,x_2) & = \expectation{p_{\phi}(\theta)}{p_{\phi}(y_1|x_1,\theta)p_{\phi}(y_2|x_2,\theta)}
    .
\end{align*}
\textbf{Active learning} \citep{atlas1989training,settles2012active} is the process of training a model on data selected by an algorithm.
In the context of supervised learning the algorithm decides which labels to acquire, commonly over a series of steps.
Each step, $t$, comprises three substeps.
First, the algorithm selects a query input, $x_t$, or a batch of inputs, often by optimising an acquisition function.
In pool-based active learning \citep{lewis1994sequential}, the setting we focus on in this work, the set of candidate inputs is a fixed pool, $\dpool = \{x_i\}_{i=1}^{N}$.
Second, the algorithm obtains a label, $y_t$, from the true conditional label distribution, $p(y|x=x_t)$, and adds $(x_t,y_t)$ to the training dataset, $\dtrain$.
Third, the model, $p_{\phi}(y|x)$, is updated on $\dtrain$.

\textbf{Bayesian experimental design} \citep{chaloner1995bayesian,lindley1956measure,rainforth2024modern} is a rigorous framework for identifying what data we expect to yield the most information gain in a quantity of interest, $\psi$.
Given a joint distribution, $p(y,\psi|x)$, the information gain in $\psi$ due to new data, $(x,y)$, is the reduction in Shannon entropy in $\psi$ that results from updating on $(x,y)$.
For a given $x$ the expected information gain over possible realisations of $y$ is
\begin{align*}
    \mathrm{EIG}_{\psi}(x) = \expectation{p_\psi(y|x)}{\entropy{p(\psi)} - \entropy{p(\psi|x,y)}}
\end{align*}
where $p_\psi(y|x) = \expectation{p(\psi)}{p(y|x,\psi)}$ is the marginal predictive distribution and $p(\psi|x,y) \propto p(y,\psi|x)$ is the posterior that results from a Bayesian update on $(x,y)$.
This quantity can be recognised as $\mutualinfo{\psi;y|x}$, the mutual information between $\psi$ and $y$ given $x$.

\textbf{Bayesian active learning} \citep{gal2017deep,houlsby2011bayesian,kirsch2023advancing,mackay1992evidence,mackay1992information} is a form of active learning that uses acquisition functions derived from the framework of Bayesian experimental design.
Traditionally it has targeted information gain in the model parameters, $\theta$, by using the BALD score \citep{houlsby2011bayesian}, $\mathrm{BALD}(x) = \mathrm{EIG}_{\theta}(x)$, as the acquisition function.
Recent work \citep{bickfordsmith2023prediction} noted the typically prediction-oriented nature of machine learning and showed the benefit of instead targeting information gain in $y_*$, the model's prediction on a random target input, $x_*$.
This yields the expected predictive information gain, $\mathrm{EPIG}(x)=$
\begin{align*}
    \expectation{p_*(x_*) p_\phi(y|x)}{\entropy{p_\phi(y_*|x_*)} - \entropy{p_\phi(y_*|x_*,x,y)}}
    .
\end{align*}
This is equivalent to $\mutualinfo{(x_*,y_*);y|x}$, the mutual information between $(x_*,y_*)$ and $y$ given $x$.
Whereas BALD targets a global reduction in parameter uncertainty, EPIG targets only reductions in parameter uncertainty that correspond to reduced predictive uncertainty on $x_*$.
As an expected value over target inputs, EPIG explicitly accounts for the input distribution.

\textbf{Unsupervised pretraining} \citep{bengio2006greedy,bommasani2021opportunities,hinton2006fast,ranzato2006efficient} is a way to distil information from unlabelled data into a model, with the aim of improving the model's performance on downstream tasks.
It is common practice to construct a model composed of two parts: an encoder that maps a given input to a latent representation, and a prediction head that maps the latent representation to a predictive distribution.
The encoder is trained on unlabelled data (pretraining), then either the whole model or just the prediction head is trained on labelled data (finetuning).
This two-step training scheme produces a semi-supervised model, incorporating both labelled and unlabelled data.
\section{Problems with current models}
\label{sec:problem}

Unlabelled data can provide a significant amount of information helpful for downstream predictive tasks, as demonstrated by cases like few-shot learning in vision models \citep{chen2020big} and in-context learning in language models \citep{brown2020language}.
A key aspect of this information relates to the underlying input manifold, knowledge of which lets the model make better similarity judgements between inputs and thus produce more appropriate predictive correlations.
Such correlations are crucial to downstream uncertainty estimation \citep{osband2022neural,osband2022evaluating,wang2021beyond}.

The fully supervised models widely used in Bayesian active learning---examples include supervised deep ensembles \citep{beluch2018power,chitta2018large} and supervised dropout networks \citep{atighehchian2020bayesian,gal2017deep,kirsch2019batchbald,kirsch2023speeding}---cannot directly capture the information present in unlabelled data, as inputs without labels do not convey information under such models.
A consequence of this is well documented in multiple studies of other active-learning approaches (see \Cref{sec:related_work}): the model suffers a direct reduction in predictive performance.

Another consequence, one specific to Bayesian active learning, has been overlooked in past work: not accounting for unlabelled data undermines our estimation of the model's reducible uncertainty and therefore poses a challenge to acquiring informative labels.
One perspective on this is that the model's predictive correlations are much worse than they would be if unlabelled data were incorporated.
An alternative (and related) perspective is that failing to incorporate known relevant information into the model from the start leads to reducible uncertainty in the model that need not exist, which in turn leads to spending precious labels gathering the very same information.

The problem here is particularly pertinent when using big models and assuming all the model parameters will be updated after each data-acquisition step.
This is partly because big models complicate prior specification.
In principle we can incorporate information from unlabelled data into a fully supervised model by analysing the data and encoding our insights into the prior beliefs that we specify \citep{gelman2013bayesian,kruschke2014doing}.
But working with big models and lots of unlabelled data often makes this infeasible \citep{fortuin2021priors,tran2022all,wenzel2020good}.

On top of the issue of prior specification is that of non-Bayesian updating on new data.
Whenever we use a form of reducible uncertainty, such as BALD or EPIG, we implicitly assume new data will be incorporated using an exact Bayesian update \citep{rainforth2024modern}.
Yet the high parameter counts seen in practice typically necessitate updating schemes that do not match the Bayesian ideal, including those based on inaccurate inference approximations \citep{farquhar2022understanding,sharma2023bayesian}.
This introduces a mismatch between the assumed and actual updates, reducing the correspondence between how informative new data is estimated to be and how informative it actually is.
Targeting the best data can thus become challenging.

A further problem is the computational cost of active-learning steps.
Updates on new data can be expensive even if they are not Bayesian.
Estimating the model's reducible uncertainty for data-acquisition decisions is also potentially costly.
While the extra overall computational cost is often justified if it allows us to more wisely allocate our labelling budget, it might be too great for many practical applications, especially when performing batch acquisition \citep{kirsch2019batchbald}.

\begin{figure}[t!]
    \centering
    \includegraphics[height=3.33\figureheight,trim={0.5cm 0 0 0}]{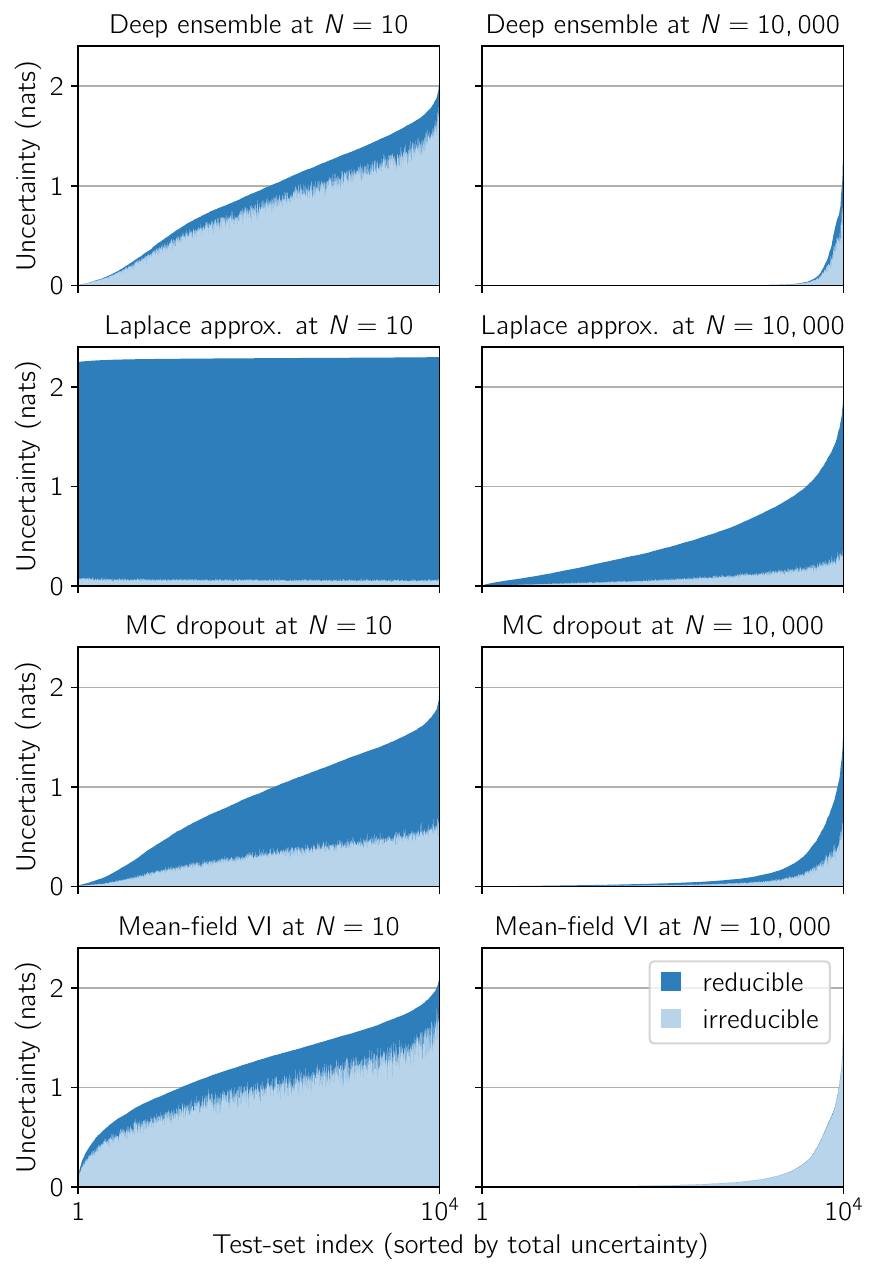}
    \caption{
        Bayesian deep learning can struggle to appropriately decompose a model's estimated uncertainty.
        Most of the uncertainty of the model here is ultimately reducible, as demonstrated by the decrease in total uncertainty between training on a small amount of data ($N=10$) and training on more data ($N=10,000$).
        This suggests the irreducible-uncertainty estimates produced by these methods are inconsistent.
        Using deep ensembles, Monte Carlo dropout and mean-field variational inference, much of the uncertainty deemed to be irreducible at $N=10$ is resolved at $N=10,000$.
        Using Laplace approximation, the irreducible uncertainty increases on many of the test inputs after training on more data.
        See \Cref{sec:experiment_details} for details.
    }
    \label{fig:uncertainty_decomposition_bnn}
\end{figure}

In \Cref{fig:uncertainty_decomposition_bnn} we demonstrate the kind of problem that can occur with commonly used models and inference approximations.
Given this relatively simple MNIST classification problem \citep{lecun1998gradientbased}, we can reasonably believe that with enough training data the neural network of interest will be able to make correct, confident predictions on test inputs \citep{belkin2018overfitting,zhang2017understanding}.
In other words, most of the model's uncertainty can be reduced.
Our results reflect this: across all the methods considered, the model has high total uncertainty on many test inputs after training on a small dataset and then has much lower total uncertainty after training on a bigger one.
This reveals the uncertainty estimates to often be overly pessimistic, with a significant fraction of the model's uncertainty deemed to be irreducible.
Failing to appropriately decompose the model's uncertainty can have a critical effect on Bayesian active learning: it directly affects which inputs are prioritised for labelling.
\section{A semi-supervised approach \break to Bayesian active learning}
\label{sec:method}

We now address how to effectively incorporate unlabelled data into Bayesian active learning.
Our proposal is a simple semi-supervised framework: pretrain a deterministic encoder on unlabelled inputs, then train a stochastic prediction head through Bayesian active learning.
The idea is that the encoder (which we keep fixed) will capture much of the information present in the unlabelled data, while the prediction head will encapsulate the information we gather from labelled data and help guide acquisition by capturing a notion of the model's reducible uncertainty.

\subsection{Semi-supervised model}
\label{sec:semi_supervised_model}

The basis of our framework is to split the predictive model into an encoder, $g$, and a prediction head, $h$:
\begin{align*}
    p_\phi(y|x) = \expectation{p_\phi(\theta_h)}{p_\phi(y|g(x),\theta_h)}
\end{align*}
where $\theta_h \sim p_\phi(\theta_h)$ represents the parameters of the prediction head and $p_\phi(y|g(x),\theta_h)$ is the predictive distribution for a given configuration of $\theta_h$.

The encoder is pretrained on unlabelled data and then treated as a fixed function with no stochastic parameters (see \Cref{sec:fixing_the_encoder}).
Any pretraining method can be used, with a preference for methods that give strong downstream predictive performance.
Because the encoder is deterministic and only needs to be trained once, it can be big and have a lot of processing power while still avoiding the issues discussed in \Cref{sec:problem}.

The prediction head, in contrast, is updated as new labels are acquired and conforms to the setup described in \Cref{sec:background}.
Typically it can be significantly more lightweight than either the encoder or the fully supervised models conventionally used in Bayesian active learning.
This is partly because the amount of unlabelled data typically dwarfs the label budget, and in such settings this model setup supports good predictions \citep{assran2022masked,chen2020big,henaff2020data} and uncertainty estimation \citep{daxberger2021laplace,kristiadi2020being,sharma2023bayesian,snoek2015scalable}.
It is also because using a smaller prediction head allows cheaper and more consistent parameter updates on new data (see \Cref{sec:fixing_the_encoder}).

Importantly the model's reducible uncertainty is now dictated entirely by $p_\phi(\theta_h)$.
For example, BALD, which measures the total reducible uncertainty in the model parameters, can be written as $\mathrm{BALD}(x) = \mutualinfo{(\theta_g,\theta_h);y|x} = \mutualinfo{\theta_h;y|x}$, due to the fact that the encoder parameters, $\theta_g$, are not random.

\subsection{Fixing the encoder}
\label{sec:fixing_the_encoder}

A key decision when using a pretrained encoder is whether to keep it fixed or finetune it along with the prediction head.
We suggest that fixing the encoder is a sensible default in our framework (though occasionally updating it using semi-supervised representation learning might sometimes be helpful).
Our reasons for this draw on insights from Bayesian experimental design along with more practical considerations.

\subsubsection{Update consistency}
\label{sec:update_consistency}

In \Cref{sec:problem} we explained that data-acquisition decisions can be undermined by a mismatch between the exact Bayesian update assumed by formal notions of reducible uncertainty and the actual update made in practice.
Finetuning the encoder thus introduces a difficult technical issue: regardless of how we do it, it substantially increases the scope for a mismatch between assumed and actual updates.
On the one hand we could use a deterministic encoder and update some or all of its parameters in each active-learning step.
Here the update mismatch would be because only distributional updates on the model's stochastic parameters are accounted for when estimating the model's reducible uncertainty.
On the other hand we could use a stochastic encoder and update it through approximate inference.
Here the update mismatch would be due to inaccuracies in the approximation scheme.

The model setup we propose can help mitigate this issue.
It makes consistent updating more viable by focusing updates on the lightweight prediction head.

\subsubsection{Computational cost}
\label{sec:computational_cost}

Bayesian active learning becomes considerably cheaper computationally when we use a powerful fixed encoder.
One reason for this is that we only need to make updates to $p_\phi(\theta_h)$ at each active-learning step, with any associated inference much easier for this small subset of the model parameters than for all of them.
Another key reason is that we reduce the number of encoder forward passes we have to compute: we only need to pass the inputs through the encoder once at the start of active learning to precompute $\{g(x_i)\}_{i=1}^N$, whereas in the conventional setup we need to evaluate the full model each time we use $p_\phi(y|x,\theta)$.

This reduction in computational cost can matter a lot in practice.
Training a powerful predictive model tends to require a lot of computation.
The pretraining step in our framework amortises this computation, shifting it outside the active-learning loop.
We can thus get the benefit of active learning---the ability to spend extra computation to deal with high labelling costs---without introducing big delays as the model's parameter distribution updates or as its uncertainty is estimated.
This opens up use cases that otherwise might be impractical, such as human-in-the-loop labelling.

\subsection{Acquisition function}
\label{sec:acquisition_function}

While it is well established that using semi-supervised models leads to better predictive performance, some past empirical studies have raised doubts about the benefit of active learning in these models (see \Cref{sec:related_work}).
This appears to contradict our argument that incorporating information from unlabelled data should lead to better decisions about what data to acquire.
A partial explanation for this apparent contradiction is the data used in these studies.
The need for active learning is clearest when there is a big disparity between the most and least relevant inputs to label.
Past studies have commonly used carefully curated pools of unlabelled data, which provide an unrepresentative boost to the performance of random acquisition.

An important additional explanation is that previous evaluations of Bayesian active learning have overwhelmingly focused on BALD acquisition \citep{atighehchian2020bayesian,beluch2018power,burkhardt2018semi,chitta2018large,gal2017deep,houlsby2011bayesian,jeon2020thompsonbald,kirsch2019batchbald,kirsch2023black,kirsch2023speeding,kirsch2023stochastic,lee2019baldvae,murray2021depth,osband2023fine,pop2018deep,tran2019bayesian}.
Recent work showed that BALD can exhibit serious pathologies when our aim is to maximise predictive performance, motivating the use of EPIG as an alternative \citep{bickfordsmith2023prediction}.
In our own experiments we often find EPIG to be a key contributor to the success of semi-supervised Bayesian active learning: BALD can perform worse than random acquisition, while EPIG produces a reliable improvement.

Crucial to this success, we believe, is that EPIG directly targets information about downstream predictions.
Adapting a result from \citet{bernardo1979expected} clarifies the relationship between the two acquisition functions.

\begin{proposition}
    \label{thm:bald_epig_inequality}
    $\mathrm{EPIG}(x) \leq \mathrm{BALD}(x)$, with equality if $y_*$ is a one-to-one function of $\theta$ for all $x_*$.
\end{proposition}

\begin{proof}
    Using the chain rule for mutual information, we can equate two ways of writing $\mutualinfo{y;y_*,\theta|x,x_*}$:
    \begin{align*}
         & \mutualinfo{y;\theta|x,x_*} + \expectation{p_\phi(\theta)}{\mutualinfo{y;y_*|x,x_*,\theta}}        \\
         & \qquad = \mutualinfo{y;y_*|x,x_*} + \expectation{p_\phi(y_*|x_*)}{\mutualinfo{y;\theta|x,x_*,y_*}}
        .
    \end{align*}
    Now we note that $\mutualinfo{y;y_*|x,x_*,\theta}=0$ because the model's predictions are independent if we condition on $\theta$.
    We also note that $\mutualinfo{y;\theta|x,x_*,y_*}\geq 0$, so
    \begin{align*}
        \mutualinfo{y;\theta|x,x_*} - \mutualinfo{y;y_*|x,x_*}
         & \geq 0.
    \end{align*}
    Taking an expectation over $p(x_*)$ and noting that $\mutualinfo{y;\theta|x,x_*}=\mutualinfo{y;\theta|x}=\mathrm{BALD}(x)$ now gives the desired inequality.
    This in turn gives an equality if $y_*$ is a one-to-one function of $\theta$ for all $x_*$ because that implies $\mutualinfo{y;\theta|x,x_*} = \mutualinfo{y;y_*|x,x_*}$, since the mutual information between two variables is invariant to invertible transformations \citep{cover2005elements}.
\end{proof}

With this we can interpret EPIG as a ``filtered'' version of BALD that only retains information relevant to the model's predictions on the task of interest.
As a result it avoids BALD's pathology of acquiring data that is informative with respect to the model's parameters but not its downstream predictions.
EPIG's role in our framework can therefore be seen as complementary to that of unsupervised pretraining: pretraining provides breadth, capturing much of the information that might be useful, while EPIG provides the ability to hone in on the task of interest.
\section{Related work}
\label{sec:related_work}

Early work on active learning included a demonstration by \citet{mccallum1998employing} of using a semi-supervised model, specifically a naive-Bayes classifier, with query-by-committee data acquisition \citep{seung1992query}.
\citet{muslea2002active} likewise proposed using a semi-supervised naive-Bayes model along with disagreement-based acquisition.
\citet{riccardi2003active}, \citet{tur2005combining} and \citet{tomanek2009semi} considered an approach based on model confidence: select low-confidence inputs for labelling; assign pseudolabels (using the model) to high-confidence inputs and include them in the training data.
\citet{zhu2003combining} studied a Gaussian-random-field binary classifer, using an estimate of the expected loss reduction as their acquisition function.
\citet{nguyen2004active} used a cluster-based binary classifier and decided which labels to acquire based on candidate inputs' estimated contribution to the model's error.

The recent work most directly related to ours is that which includes empirical evaluations of BALD acquisition in combination with semi-supervised models. 
(EPIG has not previously been tested with such models.)
The findings of the evaluations have been mixed: some show BALD outperforming random acquisition while others show it performing the same or worse.

Some of the most positive results have been found on language datasets.
\citet{burkhardt2018semi} used a semi-supervised variational autoencoder and BALD acquisition in the context of classifying social-media posts, and demonstrated benefits both from incorporating unlabelled data and from active data acquisition.
\citet{seo2022active} combined unsupervised pretraining with active learning and found BALD gave consistent gains over random acquisition on relation-extraction and sentence-classification tasks.
\citet{osband2023fine} studied a similar method on a range of text-classification tasks and saw BALD perform strongly.

BALD has performed less convincingly in computer-vision evaluations.
\citet{mittal2023best} reported positive results on semantic-segmentation datasets.
But \citet{sener2018active}, \citet{hacohen2022active}, \citet{yehuda2022active} and \citet{luth2023toward} found BALD sometimes performed better and other times worse than random acquisition on image-classification datasets.

\begin{figure*}[t]
    \centering
    \includegraphics[height=\figureheight,trim={1cm 0 0 0}]{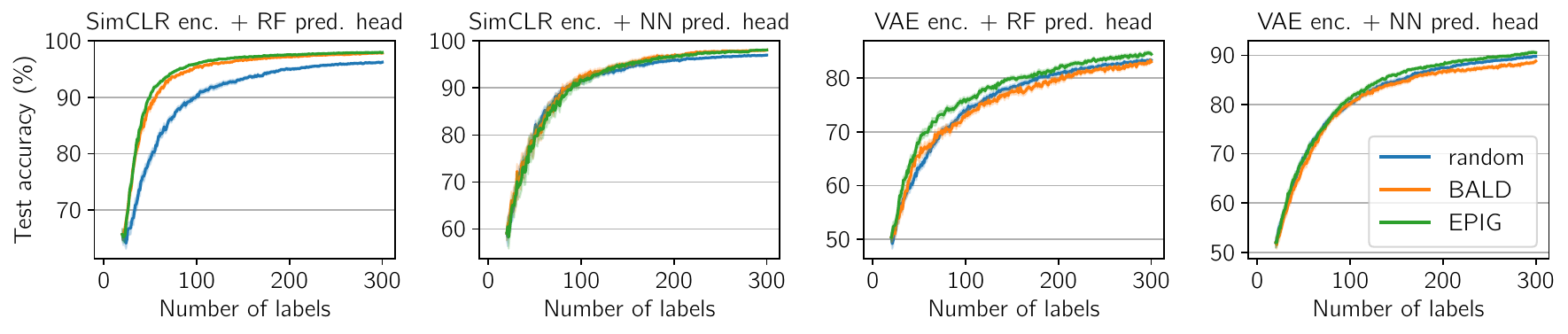}
    \caption{
        EPIG's boost over random acquisition is relatively robust to the choice of encoder and prediction head in semi-supervised models applied to MNIST data.
        BALD only convincingly beats random for two out of four configurations.
    }
    \label{fig:encoder_and_head_matter}
\end{figure*}

In the broader literature there has also been a degree of ambiguity about the benefit of active learning (relative to training on randomly acquired data) when using semi-supervised models.
A number of methods for semi-supervised active learning have been found to beat non-active counterparts in evaluations focusing on computer vision
\citep{bhatnagar2020pal,chu2022active,ebrahimi2020minimax,gao2020consistency,gudovskiy2020deep,guo2016semi,guo2021redundancy,guo2021semi,hacohen2022active,huang2021semi,luth2023toward,margatina2022importance,mittal2023best,mohamadi2022deep,mottaghi2019adversarial,pourahmadi2021simple,rhee2017active,rottmann2018deep,sener2018active,shui2020deep,tamkin2022active,wang2020dual,yang2015multi,yehuda2022active,yi2022using,zhang2020state,zhang2024labelbench}, natural-language processing \citep{eindor2020active,maekawa2022low,seo2022active,tamkin2022active,yuan2020cold}, molecular-property prediction \citep{hao2020asgn,zhang2019bayesian} and speech recognition \citep{drugman2016active}.
Methods notable for their simplicity and their strong performance on standard problems include three that use unsupervised pretraining and that base data acquisition on different approaches to covering the input space: ProbCover \citep{yehuda2022active}, TypiClust \citep{hacohen2022active} and a method based on $k$-means clustering \citep{pourahmadi2021simple}.

Alongside the studies reporting positive results have been several casting doubt on how helpful active learning is when using a semi-supervised model.
\citet{bengar2021reducing}, \citet{chan2021marginal}, \citet{mittal2019parting} and \citet{simeoni2020rethinking} all evaluated multiple acquisition methods on multiple image-classification datasets and found no acquisition method reliably outperformed random acquisition once unlabelled data had been incorporated into the model.
Likewise none of the acquisition methods tried by \citet{gleave2022uncertainty} improved over random acquisition in the context of text classification using a pretrained model.
\section{Experiments}
\label{sec:experiments}

To help resolve some of the ambiguity in the literature and to validate our proposed framework, we evaluated variants of Bayesian active learning on a range of tasks.
We focused on image classification, in line with much of the literature on foundational active-learning methodology (see \Cref{sec:related_work}).
Details about the data, models and active-learning setup we used are given in \Cref{sec:experiment_details}.
Code for reproducing our results is available at \shorturl{github.com/fbickfordsmith/epig}.

\subsection{Semi-supervised models outperform fully supervised ones but BALD is unreliable}
\label{sec:pretraining_helps}

We first set out to understand the interplay between semi-supervised models and BALD-based active learning, since past work leaves it unclear what to expect from the combination (see \Cref{sec:related_work}).
Using the MNIST \citep{lecun1998gradientbased}, dSprites \citep{matthey2017dsprites} and CIFAR-10 \citep{krizhevsky2009learning} datasets, we evaluated BALD-based active learning with fully supervised and semi-supervised models.

We find that using a semi-supervised model instead of a fully supervised one leads to better predictive performance, but also that BALD does not reliably help on top of this (see \Cref{fig:pretraining_helps}).
This finding is in line with results from multiple past studies.
A plausible explanation is that, despite the model incorporating information from unlabelled data and supporting better uncertainty estimation, the type of uncertainty reduction that BALD targets causes irrelevant data to be acquired.
The results might thus indicate a need for more targeted information-seeking.

\subsection{EPIG produces a reliable boost over both BALD and random acquisition}
\label{sec:epig_helps}

We next investigated whether acquiring data using EPIG addresses the problems associated with using BALD.
Returning to the datasets and semi-supervised models from before, we ran active learning with EPIG.

Our results show EPIG convincingly outperforming random acquisition in cases where BALD fails to (see \Cref{fig:pretraining_helps}).
Referring to our discussion in \Cref{sec:method}, we can understand this to be a result of EPIG being more specific in the information it favours, with the pretrained encoder supporting effective judgements of similarity between inputs and thus helping to target information relevant to downstream prediction.

EPIG's advantage over BALD is also clear when we consider their sensitivity to changes in the semi-supervised model.
Revisiting MNIST with changes to the encoder and prediction head, we see in \cref{fig:encoder_and_head_matter} that BALD's performance is even shakier than the results in \Cref{fig:pretraining_helps} let on: it does not convincingly outperform random acquisition when we switch from the SimCLR encoder to the VAE encoder.
EPIG shows greater robustness to changes in the model.

\subsection{The semi-supervised approach works with messy data and at scale}
\label{sec:class_imbalance_matters}

We then evaluated our approach in scenarios that capture aspects of working with messy data in practical applications, contrasting with the relatively well-curated and small-scale datasets from before.
To this end we first looked at ``redundant'' variants of MNIST and CIFAR-10, in which the pool of unlabelled data contains many more classes than the ones we want to discriminate between \citep{bickfordsmith2023prediction}.

The strong performance we see on Redundant MNIST and Redundant CIFAR-10 (see \Cref{fig:class_imbalance_matters}) has important practical implications.
Modern machine learning often entails using messy pools of unlabelled data---for example, huge collections of audio, images and text \citep{ardila2020common,gemmeke2017audio,mahajan2018exploring,radford2021learning,raffel2020exploring,sun2017revisiting}---where many of the inputs have little relevance to the task we are ultimately interested in.
Our framework provides a way to make the most of these resources while overcoming the challenges they pose.
Through unsupervised pretraining we can capitalise on the information encoded in unlabelled inputs, and through EPIG-based data acquisition we can avoid the trap of gathering labels for obscure inputs that have limited relevance to the task of interest.

\begin{figure}[t]
    \centering
    \includegraphics[height=\figureheight, trim={0.2cm 0 0 0}]{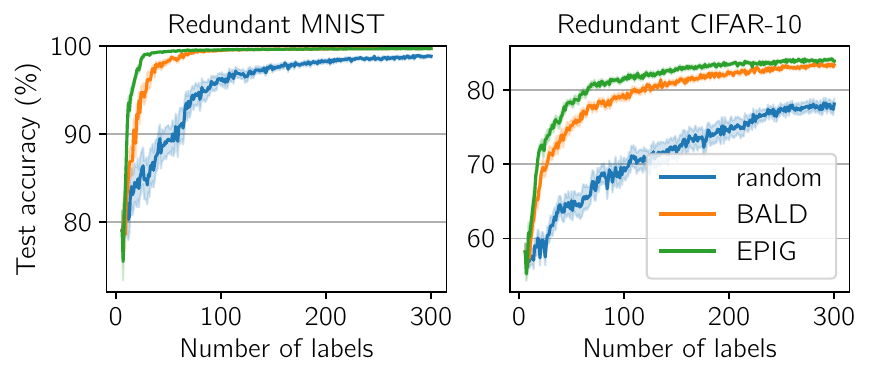}
    \caption{
        BALD and EPIG provide especially big gains in predictive performance relative to random acquisition when the pool of unlabelled data is not carefully curated.
    }
    \label{fig:class_imbalance_matters}
\end{figure}
\begin{figure}[t]
    \centering
    \includegraphics[height=\figureheight,trim={0.2cm 0 0 0}]{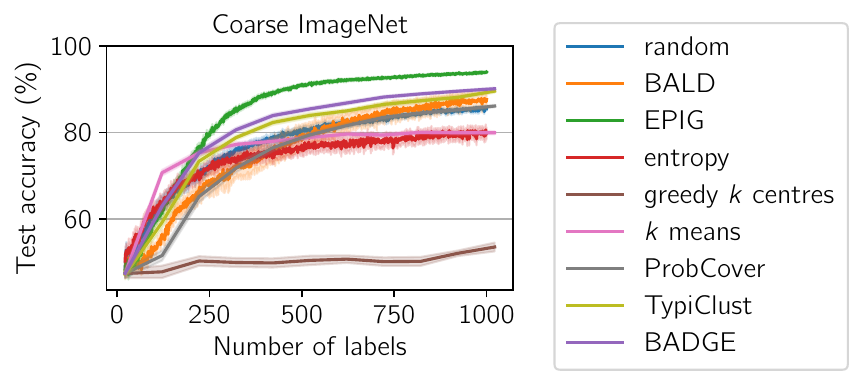}
    \caption{
        Semi-supervised Bayesian active learning performs strongly on ImageNet data, with EPIG producing a particularly notable boost over random acquisition.
    }
    \label{fig:imagenet_bayesian}
\end{figure}

We finally focused on scaling up to higher-dimensional input spaces.
In particular we looked at classifying ImageNet images into eleven superclasses based on the WordNet semantic hierarchy \citep{engstrom2019robustness,miller1995wordnet}: bird, boat, car, cat, dog, fruit, fungus, insect, monkey, truck and other.
For a better sense of context we evaluated six extra acquisition methods: maximum entropy \citep{settles2008analysis}, BADGE \citep{ash2020deep}, greedy $k$ centres \citep{sener2018active}, $k$ means \citep{pourahmadi2021simple}, ProbCover \citep{yehuda2022active} and TypiClust \citep{hacohen2022active}.
As noted in \Cref{sec:related_work}, the last three of these methods have been found to perform well in past work.
For example, ProbCover's reported performance on ImageNet data is the best we are aware of.

\Cref{fig:imagenet_bayesian} shows our approach providing outstanding performance here, to our knowledge the strongest yet observed for Bayesian active learning.
Even when compared with state-of-the-art baselines, it provides a substantial boost in test accuracy.
This demonstrates the possibility of scaling up information-theoretic data acquisition to challenging, high-dimensional inputs.
Notably this is possible provided an appropriate model and acquisition strategy are used: as before, our use of EPIG is key in achieving such strong results.
Helping to explain these gains is a notable difference in class distributions in the labelled datasets produced by random, BALD and EPIG acquisition (see \Cref{fig:imagenet_acquisitions}).

\subsection{Semi-supervised models allow much faster Bayesian active learning}
\label{sec:speedups}

A key practical finding is how much faster Bayesian active learning becomes when we use the semi-supervised approach (see \Cref{tab:speedups}).
This validates the argument we presented in \Cref{sec:fixing_the_encoder}: unsupervised pretraining gives us a powerful encoder that we can adapt to a task of interest using a lightweight prediction head, keeping the encoder fixed.
This setup allows a drastic speedup, which could unlock new practical use cases.

\begin{table}[t]
    \centering
    \scriptsize
    \begin{tabular}{lrrrr}
        \toprule
        Dataset  & Time (sup.)   & Time (semi) & Speedup    & Enc. time \\
        \midrule
        MNIST    & 32 sec        & 17 sec      & 2$\times$  & 2 msec    \\
        dSprites & 4 min 17 sec  & 16 sec      & 16$\times$ & 6 msec    \\
        CIFAR-10 & 42 min 14 sec & 29 sec      & 89$\times$ & 75 msec   \\
        \bottomrule
    \end{tabular}
    \caption{
        Semi-supervised Bayesian active learning is much faster than the conventional fully supervised approach, with the speedup increasing as we scale to harder datasets and use more powerful encoders.
        We report the mean duration of an active-learning step with EPIG acquisition, and the speedup from using the semi-supervised approach.
        To give an indication of encoder power, we report the time taken to encode a batch of 256 inputs (mean across seven runs on an Nvidia GeForce RTX 2080 Ti).
    }
    \label{tab:speedups}
\end{table}
\begin{figure*}[t]
    \centering
    \includegraphics[height=\figureheight]{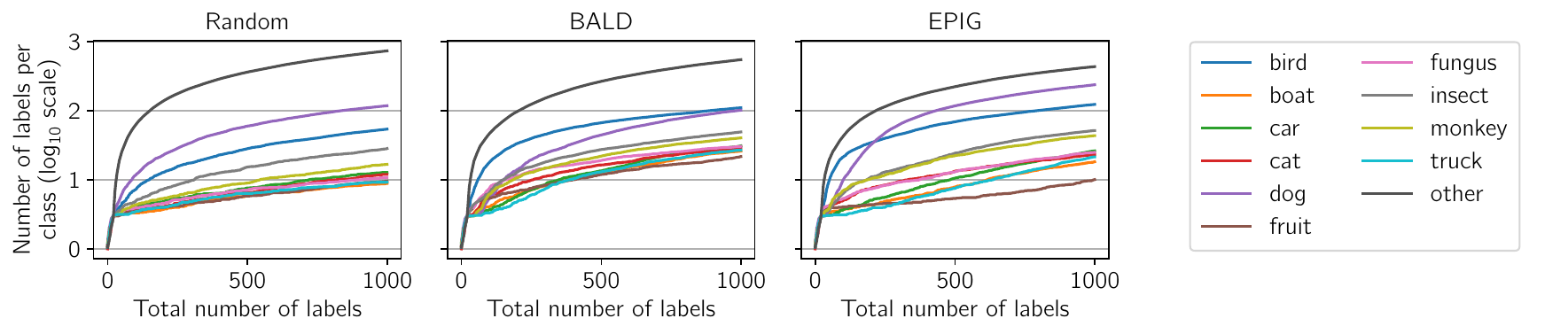}
    \caption{
        The labelled datasets generated by random, BALD and EPIG acquisition are notably distinct in their class distributions.
        Random matches the class distribution of the pool, while both BALD and EPIG differentially upweight the non-``other'' classes.
        We report the acquisitions corresponding to \Cref{fig:imagenet_bayesian} (mean label counts across random seeds).
    }
    \label{fig:imagenet_acquisitions}
\end{figure*}
\section{Conclusion}
\label{sec:conclusion}

We have shown that the fully supervised models predominantly used in Bayesian active learning are problematic.
In particular, by not accounting for the information present in unlabelled data, they undermine the pursuit of useful novel information.
We have identified a simple semi-supervised framework as an appealing solution.
While previous work has seen mixed results on the benefits of active learning with semi-supervised models, we have shown that using EPIG, a prediction-oriented acquisition function, can consistently provide improvements over random acquisition.

\vspace{10pt}

Our findings suggest a crucial update for the field.
Work on Bayesian active learning has largely been disconnected from that on semi-supervised learning.
This implicitly assumes that results from studying fully supervised models cleanly translate to semi-supervised models.
Empirical results from both prior work and our own show this assumption to be misguided.
At the same time, obtaining useful methodological insights requires studying the tools practitioners actually use, and few practitioners would neglect to make good use of unlabelled data.
Research on Bayesian active learning should embrace a shift to semi-supervised models.
\section*{Acknowledgements}

We thank Andreas Kirsch, Arnaud Doucet, Brooks Paige, Desi Ivanova, Jannik Kossen, Lewis Smith, Mike Osborne, Mrinank Sharma, Seb Farquhar and Yarin Gal for useful discussions and feedback; Ben H{\"o}ltgen, Jan Brauner and S{\"o}ren Mindermann for sharing related internal results; and the anonymous reviewers of this paper for their feedback.
Freddie Bickford Smith is supported by the EPSRC Centre for Doctoral Training in Autonomous Intelligent Machines and Systems (EP/L015897/1).
Tom Rainforth is supported by the UK EPSRC grant EP/Y037200/1.

\bibliography{references}

\begin{thebibliography}{}

\bibitem[Aitchison, 2021]{aitchison2021statistical}
Aitchison (2021).
\newblock A statistical theory of cold posteriors in deep neural networks.
\newblock {\em International Conference on Learning Representations}.

\bibitem[Ardila et~al, 2020]{ardila2020common}
Ardila, Branson, Davis, Kohler, Meyer, Henretty, Morais, Saunders, Tyers, \& Weber (2020).
\newblock {Common Voice}: a massively-multilingual speech corpus.
\newblock {\em Language Resources and Evaluation Conference}.

\bibitem[Asano et~al, 2021]{asano2021pass}
Asano, Rupprecht, Zisserman, \& Vedaldi (2021).
\newblock {PASS}: an {ImageNet} replacement for self-supervised pretraining without humans.
\newblock {\em Track on Datasets and Benchmarks, Conference on Neural Information Processing Systems}.

\bibitem[Ash et~al, 2020]{ash2020deep}
Ash, Zhang, Krishnamurthy, Langford, \& Agarwal (2020).
\newblock Deep batch active learning by diverse, uncertain gradient lower bounds.
\newblock {\em International Conference on Learning Representations}.

\bibitem[Assran et~al, 2022]{assran2022masked}
Assran, Caron, Misra, Bojanowski, Bordes, Vincent, Joulin, Rabbat, \& Ballas (2022).
\newblock Masked {Siamese} networks for label-efficient learning.
\newblock {\em European Conference on Computer Vision}.

\bibitem[Atighehchian et~al, 2020]{atighehchian2020bayesian}
Atighehchian, Branchaud-Charron, \& Lacoste (2020).
\newblock {Bayesian} active learning for production, a systematic study and a reusable library.
\newblock {\em Workshop on ``Uncertainty and Robustness in Deep Learning'', International Conference on Machine Learning}.

\bibitem[Atlas et~al, 1989]{atlas1989training}
Atlas, Cohn, \& Ladner (1989).
\newblock Training connectionist networks with queries and selective sampling.
\newblock {\em Conference on Neural Information Processing Systems}.

\bibitem[Belkin et~al, 2018]{belkin2018overfitting}
Belkin, Hsu, \& Mitra (2018).
\newblock Overfitting or perfect fitting? {Risk} bounds for classification and regression rules that interpolate.
\newblock {\em Conference on Neural Information Processing Systems}.

\bibitem[Beluch et~al, 2018]{beluch2018power}
Beluch, Genewin, N{\"u}rnberger, \& Kohler (2018).
\newblock The power of ensembles for active learning in image classification.
\newblock {\em Conference on Computer Vision and Pattern Recognition}.

\bibitem[Bengar et~al, 2021]{bengar2021reducing}
Bengar, {van de Weijer}, Twardowski, \& Raducanu (2021).
\newblock Reducing label effort: self-supervised meets active learning.
\newblock {\em International Conference on Computer Vision Workshops}.

\bibitem[Bengio et~al, 2006]{bengio2006greedy}
Bengio, Lamblin, Popovici, \& Larochelle (2006).
\newblock Greedy layer-wise training of deep networks.
\newblock {\em Conference on Neural Information Processing Systems}.

\bibitem[Bernardo, 1979]{bernardo1979expected}
Bernardo (1979).
\newblock Expected information as expected utility.
\newblock {\em Annals of Statistics}.

\bibitem[Bhatnagar et~al, 2020]{bhatnagar2020pal}
Bhatnagar, Tank, Goyal, \& Sethi (2020).
\newblock {PAL}: pretext-based active learning.
\newblock {\em arXiv}.

\bibitem[{Bickford Smith} et~al, 2023]{bickfordsmith2023prediction}
{Bickford Smith}, Kirsch, Farquhar, Gal, Foster, \& Rainforth (2023).
\newblock Prediction-oriented {Bayesian} active learning.
\newblock {\em International Conference on Artificial Intelligence and Statistics}.

\bibitem[Bommasani et~al, 2021]{bommasani2021opportunities}
Bommasani, Hudson, Adeli, Altman, Arora, {von Arx}, Bernstein, Bohg, Bosselut, Brunskill, Brynjolfsson, Buch, Card, Castellon, Chatterji, Chen, Creel, Davis, Demszky, Donahue, Doumbouya, Durmus, Ermon, Etchemendy, Ethayarajh, Fei-Fei, Finn, Gale, Gillespie, Goel, Goodman, Grossman, Guha, Hashimoto, Henderson, Hewitt, Ho, Hong, Hsu, Huang, Icard, Jain, Jurafsky, Kalluri, Karamcheti, Keeling, Khani, Khattab, Koh, Krass, Krishna, Kuditipudi, Kumar, Ladhak, Lee, Lee, Leskovec, Levent, Li, Li, Ma, Malik, Manning, Mirchandani, Mitchell, Munyikwa, Nair, Narayan, Narayanan, Newman, Nie, Niebles, Nilforoshan, Nyarko, Ogut, Orr, Papadimitriou, Park, Piech, Portelance, Potts, Raghunathan, Reich, Ren, Rong, Roohani, Ruiz, Ryan, R{\'e}, Sadigh, Sagawa, Santhanam, Shih, Srinivasan, Tamkin, Taori, Thomas, Tram{\`e}r, Wang, Wang, Wu, Wu, Wu, Xie, Yasunaga, You, Zaharia, Zhang, Zhang, Zhang, Zhang, Zheng, Zhou, \& Liang (2021).
\newblock On the opportunities and risks of foundation models.
\newblock {\em arXiv}.

\bibitem[Breiman, 2001]{breiman2001random}
Breiman (2001).
\newblock Random forests.
\newblock {\em Machine Learning}.

\bibitem[Brown et~al, 2020]{brown2020language}
Brown, Mann, Ryder, Subbiah, Kaplan, Dhariwal, Neelakantan, Shyam, Sastry, Askell, Agarwal, Herbert-Voss, Krueger, Henighan, Child, Ramesh, Ziegler, Wu, Winter, Hesse, Chen, Sigler, Litwin, Gray, Chess, Clark, Berner, McCandlish, Radford, Sutskever, \& Amodei (2020).
\newblock Language models are few-shot learners.
\newblock {\em Conference on Neural Information Processing Systems}.

\bibitem[Burgess et~al, 2017]{burgess2017understanding}
Burgess, Higgins, Pal, Matthey, Watters, Desjardins, \& Lerchner (2017).
\newblock Understanding disentangling in $\beta$-{VAE}.
\newblock {\em Workshop on ``Learning Disentangled Representations'', Conference on Neural Information Processing Systems}.

\bibitem[Burkhardt et~al, 2018]{burkhardt2018semi}
Burkhardt, Siekiera, \& Kramer (2018).
\newblock Semi-supervised {Bayesian} active learning for text classification.
\newblock {\em Workshop on ``Bayesian Deep Learning'', Conference on Neural Information Processing Systems}.

\bibitem[Chaloner \& Verdinelli, 1995]{chaloner1995bayesian}
Chaloner \& Verdinelli (1995).
\newblock {Bayesian} experimental design: a review.
\newblock {\em Statistical Science}.

\bibitem[Chan et~al, 2021]{chan2021marginal}
Chan, Li, \& Oymak (2021).
\newblock On the marginal benefit of active learning: does self-supervision eat its cake?
\newblock {\em International Conference on Acoustics, Speech and Signal Processing}.

\bibitem[Chen et~al, 2020a]{chen2020improved}
Chen, Fan, Girshick, \& He (2020a).
\newblock Improved baselines with momentum contrastive learning.
\newblock {\em arXiv}.

\bibitem[Chen et~al, 2020b]{chen2020simple}
Chen, Kornblith, Norouzi, \& Hinton (2020b).
\newblock A simple framework for contrastive learning of visual representations.
\newblock {\em International Conference on Machine Learning}.

\bibitem[Chen et~al, 2020c]{chen2020big}
Chen, Kornblith, Swersky, Norouzi, \& Hinton (2020c).
\newblock Big self-supervised models are strong semi-supervised learners.
\newblock {\em Conference on Neural Information Processing Systems}.

\bibitem[Chen et~al, 2018]{chen2018isolating}
Chen, Li, Grosse, \& Duvenaud (2018).
\newblock Isolating sources of disentanglement in variational autoencoders.
\newblock {\em Conference on Neural Information Processing Systems}.

\bibitem[Chitta et~al, 2018]{chitta2018large}
Chitta, {\'A}lvarez, \& Lesnikowski (2018).
\newblock Large-scale visual active learning with deep probabilistic ensembles.
\newblock {\em arXiv}.

\bibitem[Chu et~al, 2022]{chu2022active}
Chu, Chiang, Emam, Czaja, Leapman, Goldblum, \& Goldstein (2022).
\newblock Active learning at the {ImageNet} scale.
\newblock {\em OpenReview}.

\bibitem[Cover \& Thomas, 2005]{cover2005elements}
Cover \& Thomas (2005).
\newblock {\em Elements of Information Theory}.
\newblock John Wiley and Sons.

\bibitem[{da Costa} et~al, 2022]{dacosta2022sololearn}
{da Costa}, Fini, Nabi, Sebe, \& Ricci (2022).
\newblock solo-learn: a library of self-supervised methods for visual representation learning.
\newblock {\em Journal of Machine Learning Research}.

\bibitem[{da Costa-Luis}, 2019]{dacostaluis2019tqdm}
{da Costa-Luis} (2019).
\newblock tqdm: a fast, extensible progress meter for {Python} and {CLI}.
\newblock {\em Journal of Open Source Software}.

\bibitem[Daxberger et~al, 2021]{daxberger2021laplace}
Daxberger, Kristiadi, Immer, Eschenhagen, Bauer, \& Hennig (2021).
\newblock {Laplace} redux---effortless {Bayesian} deep learning.
\newblock {\em Conference on Neural Information Processing Systems}.

\bibitem[Denker \& LeCun, 1990]{denker1990transforming}
Denker \& LeCun (1990).
\newblock Transforming neural-net output levels to probability distributions.
\newblock {\em Conference on Neural Information Processing Systems}.

\bibitem[Dosovitskiy et~al, 2021]{dosovitskiy2021image}
Dosovitskiy, Beyer, Kolesnikov, Weissenborn, Zhai, Unterthiner, Dehghani, Minderer, Heigold, Gelly, Uszkoreit, \& Houlsby (2021).
\newblock An image is worth 16x16 words: transformers for image recognition at scale.
\newblock {\em International Conference on Learning Representations}.

\bibitem[Drugman et~al, 2016]{drugman2016active}
Drugman, Pylkk{\"o}nen, \& Kneser (2016).
\newblock Active and semi-supervised learning in {ASR}: benefits on the acoustic and language models.
\newblock {\em Interspeech}.

\bibitem[Dubois et~al, 2019]{dubois2019disentangling}
Dubois, Kastanos, Lines, \& Melman (2019).
\newblock Disentangling {VAE}.
\newblock \shorturl{github.com/yanndubs/disentangling-vae}.

\bibitem[Ebrahimi et~al, 2020]{ebrahimi2020minimax}
Ebrahimi, Gan, Salahi, \& Darrell (2020).
\newblock Minimax active learning.
\newblock {\em arXiv}.

\bibitem[Ein-Dor et~al, 2020]{eindor2020active}
Ein-Dor, Halfon, Gera, Shnarch, Dankin, Choshen, Danilevsky, Aharonov, Katz, \& Slonim (2020).
\newblock Active learning for {BERT}: an empirical study.
\newblock {\em Conference on Empirical Methods in Natural Language Processing}.

\bibitem[Engstrom et~al, 2019]{engstrom2019robustness}
Engstrom, Ilyas, Salman, Santurkar, \& Tsipras (2019).
\newblock Robustness ({Python} library).
\newblock \shorturl{github.com/madrylab/robustness}.

\bibitem[Farquhar, 2022]{farquhar2022understanding}
Farquhar (2022).
\newblock {\em Understanding approximation for {Bayesian} inference in neural networks}.
\newblock PhD thesis, University of Oxford.

\bibitem[Fortuin, 2021]{fortuin2021priors}
Fortuin (2021).
\newblock Priors in {Bayesian} deep learning: a review.
\newblock {\em International Statistical Review}.

\bibitem[Gal \& Ghahramani, 2016]{gal2016dropout}
Gal \& Ghahramani (2016).
\newblock Dropout as a {Bayesian} approximation: representing model uncertainty in deep learning.
\newblock {\em International Conference on Machine Learning}.

\bibitem[Gal et~al, 2017]{gal2017deep}
Gal, Islam, \& Ghahramani (2017).
\newblock Deep {Bayesian} active learning with image data.
\newblock {\em International Conference on Machine Learning}.

\bibitem[Gao et~al, 2020]{gao2020consistency}
Gao, Zhang, Yu, Arik, Davis, \& Pfister (2020).
\newblock Consistency-based semi-supervised active learning: towards minimizing labeling cost.
\newblock {\em European Conference on Computer Vision}.

\bibitem[Gelman et~al, 2013]{gelman2013bayesian}
Gelman, Carlin, Stern, Dunson, Vehtari, \& Rubin (2013).
\newblock {\em {Bayesian} Data Analysis}.
\newblock CRC Press.

\bibitem[Gemmeke et~al, 2017]{gemmeke2017audio}
Gemmeke, Ellis, Freedman, Jansen, Lawrence, Moore, Plakal, \& Ritter (2017).
\newblock {Audio Set}: an ontology and human-labeled dataset for audio events.
\newblock {\em International Conference on Acoustics, Speech and Signal Processing}.

\bibitem[Gleave \& Irving, 2022]{gleave2022uncertainty}
Gleave \& Irving (2022).
\newblock Uncertainty estimation for language reward models.
\newblock {\em arXiv}.

\bibitem[Gudovskiy et~al, 2020]{gudovskiy2020deep}
Gudovskiy, Hodgkinson, Yamaguchi, \& Tsukizawa (2020).
\newblock Deep active learning for biased datasets via {Fisher} kernel self-supervision.
\newblock {\em Conference on Computer Vision and Pattern Recognition}.

\bibitem[Guo et~al, 2016]{guo2016semi}
Guo, Ding, Gao, \& Wang (2016).
\newblock Semi-supervised active learning with cross-class sample transfer.
\newblock {\em International Joint Conference on Artificial Intelligence}.

\bibitem[Guo et~al, 2021a]{guo2021redundancy}
Guo, Pang, Sun, Li, \& Chen (2021a).
\newblock Redundancy removal adversarial active learning based on norm online uncertainty indicator.
\newblock {\em Computational Intelligence and Neuroscience}.

\bibitem[Guo et~al, 2021b]{guo2021semi}
Guo, Shi, Kang, Kuang, Tang, Jiang, Sun, Wu, \& Zhuang (2021b).
\newblock Semi-supervised active learning for semi-supervised models: exploit adversarial examples with graph-based virtual labels.
\newblock {\em International Conference on Computer Vision}.

\bibitem[Hacohen et~al, 2022]{hacohen2022active}
Hacohen, Dekel, \& Weinshall (2022).
\newblock Active learning on a budget: opposite strategies suit high and low budgets.
\newblock {\em International Conference on Machine Learning}.

\bibitem[Hao et~al, 2020]{hao2020asgn}
Hao, Lu, Huang, Wang, Hu, Liu, Chen, \& Lee (2020).
\newblock {ASGN}: an active semi-supervised graph neural network for molecular property prediction.
\newblock {\em International Conference on Knowledge Discovery and Data Mining}.

\bibitem[Harris et~al, 2020]{harris2020array}
Harris, Millman, {van der Walt}, Gommers, Virtanen, Cournapeau, Wieser, Taylor, Berg, Smith, Kern, Picus, Hoyer, {van Kerkwijk}, Brett, Haldane, {Fernandez del Rio}, Wiebe, Peterson, Gerard-Marchant, Sheppard, Reddy, Weckesser, Abbasi, Gohlke, \& Oliphant (2020).
\newblock Array programming with {NumPy}.
\newblock {\em Nature}.

\bibitem[He et~al, 2015]{he2015deep}
He, Zhang, Ren, \& Sun (2015).
\newblock Deep residual learning for image recognition.
\newblock {\em Conference on Computer Vision and Pattern Recognition}.

\bibitem[H{\'e}naff et~al, 2020]{henaff2020data}
H{\'e}naff, Srinivas, {De Fauw}, Razavi, Doersch, Eslami, \& {van den Oord} (2020).
\newblock Data-efficient image recognition with contrastive predictive coding.
\newblock {\em International Conference on Machine Learning}.

\bibitem[Higgins et~al, 2017]{higgins2017beta}
Higgins, Matthey, Pal, Burgess, Glorot, Botvinick, Mohamed, \& Lerchner (2017).
\newblock $\beta$-{VAE}: learning basic visual concepts with a constrained variational framework.
\newblock {\em International Conference on Learning Representations}.

\bibitem[Hinton et~al, 2006]{hinton2006fast}
Hinton, Osindero, \& Teh (2006).
\newblock A fast learning algorithm for deep belief nets.
\newblock {\em Neural Computation}.

\bibitem[Hinton \& {van Camp}, 1993]{hinton1993keeping}
Hinton \& {van Camp} (1993).
\newblock Keeping neural networks simple by minimizing the description length of the weights.
\newblock {\em Conference on Computational Learning Theory}.

\bibitem[Houlsby et~al, 2011]{houlsby2011bayesian}
Houlsby, Husz{\'a}r, Ghahramani, \& Lengyel (2011).
\newblock {Bayesian} active learning for classification and preference learning.
\newblock {\em arXiv}.

\bibitem[Huang et~al, 2021]{huang2021semi}
Huang, Wang, Xiong, Huan, \& Dou (2021).
\newblock Semi-supervised active learning with temporal output discrepancy.
\newblock {\em International Conference on Computer Vision}.

\bibitem[Hunter, 2007]{hunter2007matplotlib}
Hunter (2007).
\newblock Matplotlib: a {2D} graphics environment.
\newblock {\em Computing in Science \& Engineering}.

\bibitem[Jeon, 2020]{jeon2020thompsonbald}
Jeon (2020).
\newblock {ThompsonBALD}: a new approach to {Bayesian} batch active learning for deep learning via {Thompson} sampling.
\newblock Master's thesis, University College London.

\bibitem[Johnson et~al, 2019]{johnson2019billion}
Johnson, Douze, \& J{\'e}gou (2019).
\newblock Billion-scale similarity search with {GPUs}.
\newblock {\em IEEE Transactions on Big Data}.

\bibitem[Kim \& Mnih, 2018]{kim2018disentangling}
Kim \& Mnih (2018).
\newblock Disentangling by factorising.
\newblock {\em International Conference on Machine Learning}.

\bibitem[Kingma et~al, 2015]{kingma2015variational}
Kingma, Salimans, \& Welling (2015).
\newblock Variational dropout and the local reparameterization trick.
\newblock {\em Conference on Neural Information Processing Systems}.

\bibitem[Kingma \& Welling, 2014]{kingma2014auto}
Kingma \& Welling (2014).
\newblock Auto-encoding variational {Bayes}.
\newblock {\em International Conference on Learning Representations}.

\bibitem[Kirsch, 2023a]{kirsch2023advancing}
Kirsch (2023a).
\newblock {\em Advancing deep active learning and data subset selection: unifying principles with information-theory intuitions}.
\newblock PhD thesis, University of Oxford.

\bibitem[Kirsch, 2023b]{kirsch2023black}
Kirsch (2023b).
\newblock Black-box batch active learning for regression.
\newblock {\em Transactions on Machine Learning Research}.

\bibitem[Kirsch, 2023c]{kirsch2023speeding}
Kirsch (2023c).
\newblock Speeding up {BatchBALD}: a k-{BALD} family of approximations for active learning.
\newblock {\em arXiv}.

\bibitem[Kirsch et~al, 2023]{kirsch2023stochastic}
Kirsch, Farquhar, Atighehchian, Jesson, Branchaud-Charron, \& Gal (2023).
\newblock Stochastic batch acquisition: a simple baseline for deep active learning.
\newblock {\em Transactions on Machine Learning Research}.

\bibitem[Kirsch et~al, 2021]{kirsch2021test}
Kirsch, Rainforth, \& Gal (2021).
\newblock Test distribution-aware active learning: a principled approach against distribution shift and outliers.
\newblock {\em arXiv}.

\bibitem[Kirsch et~al, 2019]{kirsch2019batchbald}
Kirsch, van Amersfoort, \& Gal (2019).
\newblock {BatchBALD}: efficient and diverse batch acquisition for deep {Bayesian} active learning.
\newblock {\em Conference on Neural Information Processing Systems}.

\bibitem[Kluyver et~al, 2016]{kluyver2016jupyter}
Kluyver, Ragan-Kelley, Perez, Granger, Bussonnier, Frederic, Kelley, Hamrick, Grout, Corlay, Ivanov, Avila, Abdalla, Willing, \& {Jupyter development team} (2016).
\newblock Jupyter notebooks---a publishing format for reproducible computational workflows.
\newblock {\em Positioning and Power in Academic Publishing: Players, Agents and Agendas}.

\bibitem[Kristiadi et~al, 2020]{kristiadi2020being}
Kristiadi, Hein, \& Hennig (2020).
\newblock Being {Bayesian}, even just a bit, fixes overconfidence in relu networks.
\newblock {\em International Conference on Machine Learning}.

\bibitem[Krizhevsky, 2009]{krizhevsky2009learning}
Krizhevsky (2009).
\newblock Learning multiple layers of features from tiny images.
\newblock Master's thesis, University of Toronto.

\bibitem[Kruschke, 2014]{kruschke2014doing}
Kruschke (2014).
\newblock {\em Doing {Bayesian} Data Analysis: a Tutorial with {R}, {JAGS}, and {Stan}}.
\newblock Academic Press.

\bibitem[Lakshminarayanan et~al, 2017]{lakshminarayanan2017simple}
Lakshminarayanan, Pritzel, \& Blundell (2017).
\newblock Simple and scalable predictive uncertainty estimation using deep ensembles.
\newblock {\em Conference on Neural Information Processing Systems}.

\bibitem[LeCun et~al, 1998]{lecun1998gradientbased}
LeCun, Bottou, Bengio, \& Haffner (1998).
\newblock Gradient-based learning applied to document recognition.
\newblock {\em Proceedings of the IEEE}.

\bibitem[Lee \& Kim, 2019]{lee2019baldvae}
Lee \& Kim (2019).
\newblock {BALD-VAE}: generative active learning based on the uncertainties of both labeled and unlabeled data.
\newblock {\em International Conference on Robot Intelligence Technology and Applications}.

\bibitem[Lewis \& Gale, 1994]{lewis1994sequential}
Lewis \& Gale (1994).
\newblock A sequential algorithm for training text classifiers.
\newblock {\em ACM-SIGIR Conference on Research and Development in Information Retrieval}.

\bibitem[Li, 2020]{li2020bayesianize}
Li (2020).
\newblock Bayesianize: a {Bayesian} neural network wrapper in {PyTorch}.
\newblock \shorturl{github.com/microsoft/bayesianize}.

\bibitem[Lindley, 1956]{lindley1956measure}
Lindley (1956).
\newblock On a measure of the information provided by an experiment.
\newblock {\em Annals of Mathematical Statistics}.

\bibitem[L{\"u}th et~al, 2023]{luth2023toward}
L{\"u}th, Bungert, Klein, \& Jaeger (2023).
\newblock Navigating the pitfalls of active learning evaluation: a systematic framework for meaningful performance assessment.
\newblock {\em Conference on Neural Information Processing Systems}.

\bibitem[MacKay, 1992a]{mackay1992evidence}
MacKay (1992a).
\newblock The evidence framework applied to classification networks.
\newblock {\em Neural Computation}.

\bibitem[MacKay, 1992b]{mackay1992information}
MacKay (1992b).
\newblock Information-based objective functions for active data selection.
\newblock {\em Neural Computation}.

\bibitem[MacKay, 1992c]{mackay1992practical}
MacKay (1992c).
\newblock A practical {Bayesian} framework for backpropagation networks.
\newblock {\em Neural Computation}.

\bibitem[Maekawa et~al, 2022]{maekawa2022low}
Maekawa, Zhang, Kim, Rahman, \& Hruschka (2022).
\newblock Low-resource interactive active labeling for fine-tuning language models.
\newblock {\em Conference on Empirical Methods in Natural Language Processing}.

\bibitem[Mahajan et~al, 2018]{mahajan2018exploring}
Mahajan, Girshick, Ramanathan, He, Paluri, Li, Bharambe, \& {van der Maaten} (2018).
\newblock Exploring the limits of weakly supervised pretraining.
\newblock {\em European Conference on Computer Vision}.

\bibitem[Margatina et~al, 2022]{margatina2022importance}
Margatina, Barrault, \& Aletras (2022).
\newblock On the importance of effectively adapting pretrained language models for active learning.
\newblock {\em Annual Meeting of the Association for Computational Linguistics}.

\bibitem[Matthey et~al, 2017]{matthey2017dsprites}
Matthey, Higgins, Hassabis, \& Lerchner (2017).
\newblock {dSprites}: disentanglement testing sprites dataset.
\newblock \shorturl{github.com/deepmind/dsprites-dataset}.

\bibitem[McCallum \& Nigam, 1998]{mccallum1998employing}
McCallum \& Nigam (1998).
\newblock Employing {EM} and pool-based active learning for text classification.
\newblock {\em International Conference on Machine Learning}.

\bibitem[McKinney, 2010]{mckinney2010data}
McKinney (2010).
\newblock Data structures for statistical computing in {Python}.
\newblock {\em Python in Science Conference}.

\bibitem[Miller, 1995]{miller1995wordnet}
Miller (1995).
\newblock {WordNet}: a lexical database for {English}.
\newblock {\em Communications of the ACM}.

\bibitem[Mittal et~al, 2023]{mittal2023best}
Mittal, Niemeijer, Sch{\"a}fer, \& Brox (2023).
\newblock Best practices in active learning for semantic segmentation.
\newblock {\em arXiv}.

\bibitem[Mittal et~al, 2019]{mittal2019parting}
Mittal, Tatarchenko, {\c C}i{\c c}ek, \& Brox (2019).
\newblock Parting with illusions about deep active learning.
\newblock {\em arXiv}.

\bibitem[Mohamadi et~al, 2022]{mohamadi2022deep}
Mohamadi, Doretto, \& Adjeroh (2022).
\newblock Deep active ensemble sampling for image classification.
\newblock {\em Asian Conference on Computer Vision}.

\bibitem[Mottaghi \& Yeung, 2019]{mottaghi2019adversarial}
Mottaghi \& Yeung (2019).
\newblock Adversarial representation active learning.
\newblock {\em arXiv}.

\bibitem[Murray et~al, 2021]{murray2021depth}
Murray, Allingham, Antor{\'a}n, \& Hern{\'a}ndez-Lobato (2021).
\newblock Depth uncertainty networks for active learning.
\newblock {\em Workshop on ``Bayesian Deep Learning'', Conference on Neural Information Processing Systems}.

\bibitem[Muslea et~al, 2002]{muslea2002active}
Muslea, Minton, \& Knoblock (2002).
\newblock Active + semi-supervised learning = robust multi-view learning.
\newblock {\em International Conference on Machine Learning}.

\bibitem[Nguyen \& Smeulders, 2004]{nguyen2004active}
Nguyen \& Smeulders (2004).
\newblock Active learning using pre-clustering.
\newblock {\em International Conference on Machine learning}.

\bibitem[Osband et~al, 2023]{osband2023fine}
Osband, Asghari, {Van Roy}, McAleese, Aslanides, \& Irving (2023).
\newblock Fine-tuning language models via epistemic neural networks.
\newblock {\em International Conference on Machine Learning}.

\bibitem[Osband et~al, 2022a]{osband2022neural}
Osband, Wen, Asghari, Dwaracherla, Lu, Ibrahimi, Lawson, Hao, O'Donoghue, \& {Van Roy} (2022a).
\newblock The neural testbed: evaluating joint predictions.
\newblock {\em Conference on Neural Information Processing Systems}.

\bibitem[Osband et~al, 2022b]{osband2022evaluating}
Osband, Wen, Asghari, Dwaracherla, Lu, \& {Van Roy} (2022b).
\newblock Evaluating high-order predictive distributions in deep learning.
\newblock {\em Conference on Uncertainty in Artificial Intelligence}.

\bibitem[Paszke et~al, 2019]{paszke2019pytorch}
Paszke, Gross, Massa, Lerer, Bradbury, Chanan, Killeen, Lin, Gimelshein, Antiga, Desmaison, Kopf, Yang, DeVito, Raison, Tejani, Chilamkurthy, Steiner, Fang, Bai, \& Chintala (2019).
\newblock {PyTorch}: an imperative style, high-performance deep learning library.
\newblock {\em Conference on Neural Information Processing Systems}.

\bibitem[Pedregosa et~al, 2011]{pedregosa2011scikitlearn}
Pedregosa, Varoquaux, Gramfort, Michel, Thirion, Grisel, Blondel, Prettenhofer, Weiss, Dubourg, VanderPlas, Passos, Cournapeau, Brucher, Perrot, \& Duchesnay (2011).
\newblock Scikit-learn: machine learning in {Python}.
\newblock {\em Journal of Machine Learning Research}.

\bibitem[Pinsler et~al, 2019]{pinsler2019bayesian}
Pinsler, Gordon, Nalisnick, \& Hern{\'a}ndez-Lobato (2019).
\newblock {Bayesian} batch active learning as sparse subset approximation.
\newblock {\em Conference on Neural Information Processing Systems}.

\bibitem[Pop \& Fulop, 2018]{pop2018deep}
Pop \& Fulop (2018).
\newblock Deep ensemble {Bayesian} active learning.
\newblock {\em Workshop on ``Bayesian Deep Learning'', Conference on Neural Information Processing Systems}.

\bibitem[Pourahmadi et~al, 2021]{pourahmadi2021simple}
Pourahmadi, Nooralinejad, \& Pirsiavash (2021).
\newblock A simple baseline for low-budget active learning.
\newblock {\em arXiv}.

\bibitem[Radford et~al, 2021]{radford2021learning}
Radford, Kim, Hallacy, Ramesh, Goh, Agarwal, Sastry, Askell, Mishkin, Clark, Krueger, \& Sutskever (2021).
\newblock Learning transferable visual models from natural language supervision.
\newblock {\em International Conference on Machine Learning}.

\bibitem[Raffel et~al, 2020]{raffel2020exploring}
Raffel, Shazeer, Roberts, Lee, Narang, Matena, Zhou, Li, \& Liu (2020).
\newblock Exploring the limits of transfer learning with a unified text-to-text transformer.
\newblock {\em Journal of Machine Learning Research}.

\bibitem[Rainforth et~al, 2024]{rainforth2024modern}
Rainforth, Foster, Ivanova, \& {Bickford Smith} (2024).
\newblock Modern {Bayesian} experimental design.
\newblock {\em Statistical Science}.

\bibitem[Ranzato et~al, 2006]{ranzato2006efficient}
Ranzato, Poultney, Chopra, \& LeCun (2006).
\newblock Efficient learning of sparse representations with an energy-based model.
\newblock {\em Conference on Neural Information Processing Systems}.

\bibitem[Rezende et~al, 2014]{rezende2014stochastic}
Rezende, Mohamed, \& Wierstra (2014).
\newblock Stochastic backpropagation and approximate inference in deep generative models.
\newblock {\em International Conference on Machine Learning}.

\bibitem[Rhee et~al, 2017]{rhee2017active}
Rhee, Erdenee, Kyun, Ahmed, \& Jin (2017).
\newblock Active and semi-supervised learning for object detection with imperfect data.
\newblock {\em Cognitive Systems Research}.

\bibitem[Riccardi \& Hakkani-T{\"u}r, 2003]{riccardi2003active}
Riccardi \& Hakkani-T{\"u}r (2003).
\newblock Active and unsupervised learning for automatic speech recognition.
\newblock {\em European Conference on Speech Communication and Technology}.

\bibitem[Rottmann et~al, 2018]{rottmann2018deep}
Rottmann, Kahl, \& Gottschalk (2018).
\newblock Deep {Bayesian} active semi-supervised learning.
\newblock {\em International Conference on Machine Learning and Applications}.

\bibitem[Russakovsky et~al, 2014]{russakovsky2014imagenet}
Russakovsky, Deng, Su, Krause, Satheesh, Ma, Huang, Karpathy, Khosla, Bernstein, Berg, \& Fei-Fei (2014).
\newblock {ImageNet} large scale visual recognition challenge.
\newblock {\em International Journal of Computer Vision}.

\bibitem[Sener \& Savarese, 2018]{sener2018active}
Sener \& Savarese (2018).
\newblock Active learning for convolutional neural networks: a core-set approach.
\newblock {\em International Conference on Learning Representations}.

\bibitem[Seo et~al, 2022]{seo2022active}
Seo, Kim, Ahn, \& Lee (2022).
\newblock Active learning on pre-trained language model with task-independent triplet loss.
\newblock {\em AAAI Conference on Artificial Intelligence}.

\bibitem[Settles, 2012]{settles2012active}
Settles (2012).
\newblock {\em Active Learning}.
\newblock Morgan and Claypool.

\bibitem[Settles \& Craven, 2008]{settles2008analysis}
Settles \& Craven (2008).
\newblock An analysis of active learning strategies for sequence labeling tasks.
\newblock {\em Conference on Empirical Methods in Natural Language Processing}.

\bibitem[Seung et~al, 1992]{seung1992query}
Seung, Opper, \& Sompolinsky (1992).
\newblock Query by committee.
\newblock {\em Workshop on Computational Learning Theory}.

\bibitem[Sharma et~al, 2023]{sharma2023bayesian}
Sharma, Farquhar, Nalisnick, \& Rainforth (2023).
\newblock Do {Bayesian} neural networks need to be fully stochastic?
\newblock {\em International Conference on Artificial Intelligence and Statistics}.

\bibitem[Shui et~al, 2020]{shui2020deep}
Shui, Zhou, Gagn{\'e}, \& Wang (2020).
\newblock Deep active learning: unified and principled method for query and training.
\newblock {\em International Conference on Artificial Intelligence and Statistics}.

\bibitem[Sim{\'e}oni et~al, 2020]{simeoni2020rethinking}
Sim{\'e}oni, Budnik, Avrithis, \& Gravier (2020).
\newblock Rethinking deep active learning: using unlabeled data at model training.
\newblock {\em International Conference on Pattern Recognition}.

\bibitem[Smith \& Gal, 2018]{smith2018understanding}
Smith \& Gal (2018).
\newblock Understanding measures of uncertainty for adversarial example detection.
\newblock {\em Conference on Uncertainty in Artificial Intelligence}.

\bibitem[Snoek et~al, 2015]{snoek2015scalable}
Snoek, Rippel, Swersky, Kiros, Satish, Sundaram, Patwary, Prabhat, \& Adams (2015).
\newblock Scalable {Bayesian} optimization using deep neural networks.
\newblock {\em International Conference on Machine Learning}.

\bibitem[Sun et~al, 2017]{sun2017revisiting}
Sun, Shrivastava, Singh, \& Gupta (2017).
\newblock Revisiting unreasonable effectiveness of data in deep learning era.
\newblock {\em International Conference on Computer Vision}.

\bibitem[Susmelj et~al, 2020]{susmelj2020lightly}
Susmelj, Heller, Wirth, Prescott, \& Ebner (2020).
\newblock Lightly.
\newblock \shorturl{github.com/lightly-ai/lightly}.

\bibitem[Tamkin et~al, 2022]{tamkin2022active}
Tamkin, Nguyen, Deshpande, Mu, \& Goodman (2022).
\newblock Active learning helps pretrained models learn the intended task.
\newblock {\em Conference on Neural Information Processing Systems}.

\bibitem[Tomanek \& Hahn, 2009]{tomanek2009semi}
Tomanek \& Hahn (2009).
\newblock Semi-supervised active learning for sequence labeling.
\newblock {\em Annual Meeting of the Association for Computational Linguistics}.

\bibitem[Tran et~al, 2019]{tran2019bayesian}
Tran, Do, Reid, \& Carneiro (2019).
\newblock {Bayesian} generative active deep learning.
\newblock {\em International Conference on Machine Learning}.

\bibitem[Tran et~al, 2022]{tran2022all}
Tran, Rossi, Milios, \& Filippone (2022).
\newblock All you need is a good functional prior for {Bayesian} deep learning.
\newblock {\em Journal of Machine Learning Research}.

\bibitem[Tur et~al, 2005]{tur2005combining}
Tur, Hakkani-T{\"u}r, \& Schapire (2005).
\newblock Combining active and semi-supervised learning for spoken language understanding.
\newblock {\em Speech Communication}.

\bibitem[Virtanen et~al, 2020]{virtanen2020scipy}
Virtanen, Gommers, Oliphant, Haberland, Reddy, Cournapeau, Burovski, Peterson, Weckesser, Bright, {van der Walt}, Brett, Wilson, Millman, Mayorov, Nelson, Jones, Kern, Larson, Carey, Polat, Feng, Moore, VanderPlas, Laxalde, Perktold, Cimrman, Henriksen, Quintero, Harris, Archibald, Ribeiro, Pedregosa, {van Mulbregt}, \& {SciPy 1.0 contributors} (2020).
\newblock {SciPy} 1.0: fundamental algorithms for scientific computing in {Python}.
\newblock {\em Nature Methods}.

\bibitem[Wang et~al, 2020]{wang2020dual}
Wang, Li, Ma, Ma, Guan, \& Zheng (2020).
\newblock Dual adversarial network for deep active learning.
\newblock {\em European Conference on Computer Vision}.

\bibitem[Wang et~al, 2021]{wang2021beyond}
Wang, Sun, \& Grosse (2021).
\newblock Beyond marginal uncertainty: how accurately can {Bayesian} regression models estimate posterior predictive correlations?
\newblock {\em International Conference on Artificial Intelligence and Statistics}.

\bibitem[Wenzel et~al, 2020]{wenzel2020good}
Wenzel, Roth, Veeling, Swiatkowski, Tran, Mandt, Snoek, Salimans, Jenatton, \& Nowozin (2020).
\newblock How good is the {Bayes} posterior in deep neural networks really?
\newblock {\em International Conference on Machine Learning}.

\bibitem[Yadan, 2019]{yadan2019hydra}
Yadan (2019).
\newblock Hydra---a framework for elegantly configuring complex applications.
\newblock \shorturl{github.com/facebookresearch/hydra}.

\bibitem[Yang et~al, 2015]{yang2015multi}
Yang, Ma, Nie, Chang, \& Hauptmann (2015).
\newblock Multi-class active learning by uncertainty sampling with diversity maximization.
\newblock {\em International Journal of Computer Vision}.

\bibitem[Yang et~al, 2021]{yang2021study}
Yang, Yau, Fei-Fei, Deng, \& Russakovsky (2021).
\newblock A study of face obfuscation in {ImageNet}.
\newblock {\em International Conference on Machine Learning}.

\bibitem[Yehuda et~al, 2022]{yehuda2022active}
Yehuda, Dekel, Hacohen, \& Weinshall (2022).
\newblock Active learning through a covering lens.
\newblock {\em Conference on Neural Information Processing Systems}.

\bibitem[Yi et~al, 2022]{yi2022using}
Yi, Seo, Park, \& Choi (2022).
\newblock Using self-supervised pretext tasks for active learning.
\newblock {\em European Conference on Computer Vision}.

\bibitem[Yuan et~al, 2020]{yuan2020cold}
Yuan, Lin, \& Boyd-Graber (2020).
\newblock Cold-start active learning through self-supervised language modeling.
\newblock {\em Conference on Empirical Methods in Natural Language Processing}.

\bibitem[Zhang et~al, 2017]{zhang2017understanding}
Zhang, Bengio, Hardt, Recht, \& Vinyals (2017).
\newblock Understanding deep learning requires rethinking generalization.
\newblock {\em International Conference on Learning Representations}.

\bibitem[Zhang et~al, 2024]{zhang2024labelbench}
Zhang, Chen, Canal, Mussmann, Das, Bhatt, Zhu, Bilmes, Du, Jamieson, \& Nowak (2024).
\newblock {LabelBench}: a comprehensive framework for benchmarking adaptive label-efficient learning.
\newblock {\em arXiv}.

\bibitem[Zhang \& Lee, 2019]{zhang2019bayesian}
Zhang \& Lee (2019).
\newblock {Bayesian} semi-supervised learning for uncertainty-calibrated prediction of molecular properties and active learning.
\newblock {\em Chemical Science}.

\bibitem[Zhang et~al, 2020]{zhang2020state}
Zhang, Li, Yang, Wang, Zha, \& Huang (2020).
\newblock State-relabeling adversarial active learning.
\newblock {\em Conference on Computer Vision and Pattern Recognition}.

\bibitem[Zhu et~al, 2003]{zhu2003combining}
Zhu, Lafferty, \& Ghahramani (2003).
\newblock Combining active learning and semi-supervised learning using {Gaussian} fields and harmonic functions.
\newblock {\em International Conference on Machine Learning}.

\end{thebibliography}

\clearpage

\section*{Checklist}

\begin{enumerate}
    \item For all models and algorithms presented, check if you include:
          \begin{enumerate}
              \item A clear description of the mathematical setting, assumptions, algorithm, and/or model.
                    \textit{Yes. See \Cref{sec:background,sec:method}.}
              \item An analysis of the properties and complexity (time, space, sample size) of any algorithm.
                    \textit{Yes. See \Cref{tab:speedups}.}
              \item (Optional) Anonymized source code, with specification of all dependencies, including external libraries.
                    \textit{Yes. See link to code in \Cref{sec:experiments}.}
          \end{enumerate}

    \item For any theoretical claim, check if you include:
          \begin{enumerate}
              \item Statements of the full set of assumptions of all theoretical results.
                    \textit{Yes. See \Cref{thm:bald_epig_inequality}.}
              \item Complete proofs of all theoretical results.
                    \textit{Yes. See \Cref{thm:bald_epig_inequality}.}
              \item Clear explanations of any assumptions.
                    \textit{Yes. See \Cref{thm:bald_epig_inequality}.}
          \end{enumerate}

    \item For all figures and tables that present empirical results, check if you include:
          \begin{enumerate}
              \item The code, data, and instructions needed to reproduce the main experimental results (either in the supplemental material or as a URL).
                    \textit{Yes. See link to code in \Cref{sec:experiments}.}
              \item All the training details (e.g., data splits, hyperparameters, how they were chosen).
                    \textit{Yes. See \Cref{sec:experiment_details}.}
              \item A clear definition of the specific measure or statistics and error bars (e.g., with respect to the random seed after running experiments multiple times).
                    \textit{Yes. See \Cref{sec:experiment_details}.}
              \item A description of the computing infrastructure used. (e.g., type of GPUs, internal cluster, or cloud provider).
                    \textit{Yes. See \Cref{sec:resources}.}
          \end{enumerate}

    \item If you are using existing assets (e.g., code, data, models) or curating/releasing new assets, check if you include:
          \begin{enumerate}
              \item Citations of the creator, if your work uses existing assets.
                    \textit{Yes. See \Cref{sec:resources}.}
              \item The license information of the assets, if applicable. \textit{Yes. See \Cref{sec:resources}.}
              \item New assets either in the supplemental material or as a URL, if applicable.
                    \textit{Yes. See link to code in \Cref{sec:experiments}.}
              \item Information about consent from data providers/curators.
                    \textit{Not applicable.}
              \item Discussion of sensible content if applicable, e.g., personally identifiable information or offensive content.
                    \textit{One of our experiments is based on ImageNet data.
                        Some images in this dataset depict people, raising concerns with respect to issues such as demographic representation and privacy.
                        We support work that aims to address these and other ethical concerns, for example the work of \citet{asano2021pass} and \citet{yang2021study}.}
          \end{enumerate}

    \item If you used crowdsourcing or conducted research with human subjects, check if you include:
          \begin{enumerate}
              \item The full text of instructions given to participants and screenshots.
                    \textit{Not applicable.}
              \item Descriptions of potential participant risks, with links to Institutional Review Board (IRB) approvals if applicable.
                    \textit{Not applicable.}
              \item The estimated hourly wage paid to participants and the total amount spent on participant compensation.
                    \textit{Not applicable.}
          \end{enumerate}
\end{enumerate}

\clearpage

\appendix
\onecolumn

\section{Extra results}
\label{sec:extra_results}

\begin{figure*}[h]
    \centering
    \includegraphics[height=1.85\figureheight]{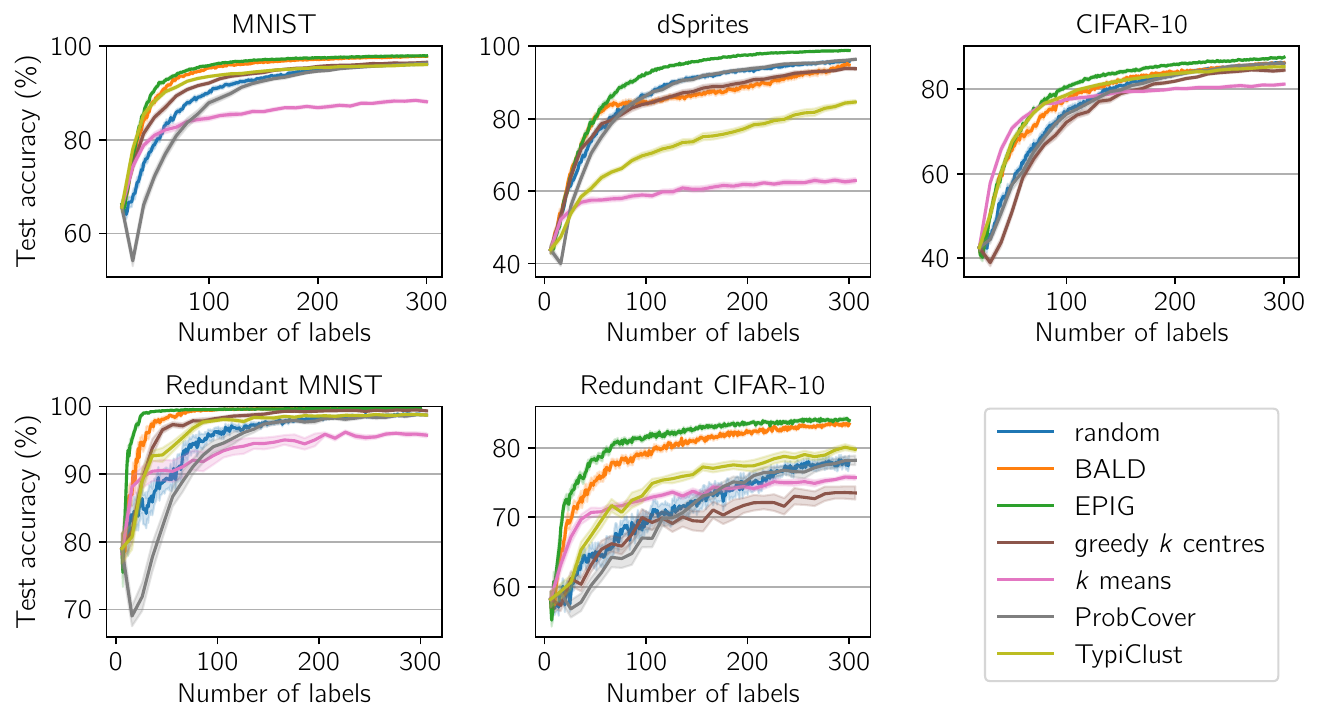}
    \caption{
        EPIG outperforms non-Bayesian acquisition methods in the settings previously covered in \Cref{fig:pretraining_helps,fig:class_imbalance_matters}.
        Unlike in \Cref{fig:imagenet_bayesian}, BADGE and maximum entropy are not shown here because the former does not by default work with random forests \citep{kirsch2023black} and the latter is equivalent to BALD when using random forests.
    }
    \label{fig:extra_baselines}
\end{figure*}
\begin{figure*}[h]
    \centering
    \includegraphics[height=\figureheight]{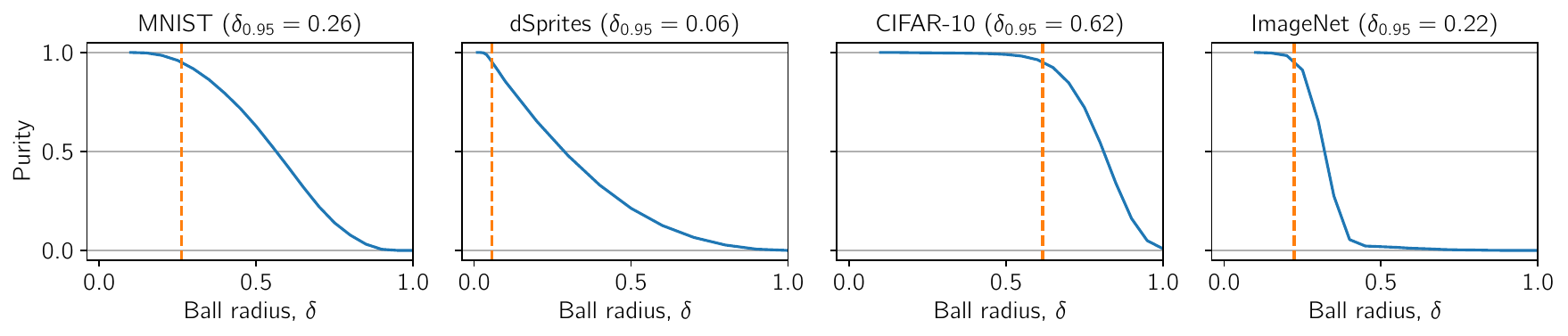}
    \caption{
        ProbCover has a ``ball radius'' parameter that needs to be tuned for the particular representation being used.
        Shown here are the results we used to choose the ball radius to ensure a ``purity'' of 0.95, in line with \citet{yehuda2022active}.
    }
    \label{fig:probcover_tuning}
\end{figure*}
\section{Experiment details}
\label{sec:experiment_details}

\subsection{Experiment in \Cref{sec:problem}}

Using MNIST data \citep{lecun1998gradientbased} and the two-convolutional-layer model architecture from \citet{kirsch2019batchbald}, we investigated the uncertainty estimates produced by four approaches to Bayesian deep learning:

\begin{enumerate}
    \item Deep ensembles \citep{lakshminarayanan2017simple} using an ensemble size of 10.

    \item Laplace approximation \citep{denker1990transforming,mackay1992practical} using a diagonal posterior covariance matrix and a standard Gaussian prior, $\mathcal{N}(0, I)$, when determining the posterior covariance.

    \item Monte Carlo dropout \citep{gal2016dropout} using a dropout rate of 0.5.

    \item Mean-field variational inference \citep{hinton1993keeping} using a Gaussian variational posterior, a standard Gaussian prior, $\mathcal{N}(0, I)$, and the local reparameterisation trick \citep{kingma2015variational} during training.
\end{enumerate}

With each approach we trained models on labelled datasets (randomly sampled and balanced by class) of two different sizes, $N=10$ and $N=10,000$, then we estimated the models' uncertainty on the MNIST test set.
For training we used the same setup as described in \Cref{sec:nn_training}, except with a validation set of 1,000 examples.
For uncertainty estimation we considered the decomposition discussed by \citet{smith2018understanding}:
\begin{align*}
    \underbrace{\mathrm{BALD}(x)}_\mathrm{reducible}
    =
    \underbrace{\entropy{\expectation{p_\phi(\theta)}{p_\phi(y|x,\theta)}}}_\mathrm{total}
    -
    \underbrace{\expectation{p_\phi(\theta)}{\entropy{p_\phi(y|x,\theta)}}}_\mathrm{irreducible}
    .
\end{align*}
While deep ensembles permit exact computation of the required expectations over $p_\phi(\theta)$, the other approaches require Monte Carlo estimation of those expectations, for which we used 1,000 parameter samples.

\subsection{Experiments in \Cref{sec:experiments}}

\subsubsection{Data}

We focused on four datasets: MNIST \citep{lecun1998gradientbased}, dSprites \citep{matthey2017dsprites}, CIFAR-10 \citep{krizhevsky2009learning} and ImageNet (ILSVRC2012) \citep{russakovsky2014imagenet}.
dSprites is less commonly studied than the others in the context of active learning but appeals because it is simple---the inputs are black-and-white images of symbols with known factors of variation, including the symbol shape (ellipse, heart or square), which we use as the class---yet harder than MNIST for fully supervised models, and a number of pretraining methods have been developed with it in mind \citep{burgess2017understanding,chen2018isolating,higgins2017beta,kim2018disentangling}.

In addition to studying the standard versions of MNIST and CIFAR-10, we looked at ``redundant'' versions in which the unlabelled pool remains as before but the task of interest involves classifying inputs from just two classes (1 vs 7 for MNIST; cat vs dog for CIFAR-10).
All the remaining classes were grouped into a third ``other'' class during acquisition and training.
At test time we treated the model as a binary classifier between the two classes of interest.
This test-time setup differs from that used by \citet{bickfordsmith2023prediction} in the original Redundant MNIST experiment; we used it because it more faithfully reflects the assumptions of the task setup.

Our ImageNet experiment focuses on classifying images into eleven superclasses based on the WordNet semantic hierarchy \citep{engstrom2019robustness,miller1995wordnet}: bird (52 classes), boat (6 classes), car (10 classes), cat (8 classes), dog (116 classes), fruit (7 classes), fungus (7 classes), insect (27 classes), monkey (13 classes), truck (7 classes) and other (747 classes).
Contrasting with our experiment based on Redundant MNIST and Redundant CIFAR-10, in this experiment we aimed to simulate scenarios where we cannot be sure that test inputs will belong to the main classes of interest.
So here we defined the test set to include all the examples from the main superclasses but additionally a small number of examples from the ``other'' superclass.
Inducing a mismatch between the pool set and the test set in terms of their class distributions, we set up the pool of candidate inputs for labelling to have an equal number of inputs from each of the 1,000 ImageNet classes.

\subsubsection{Models}

\paragraph{Main semi-supervised models}

Our proposed model setup comprises an encoder and a prediction head.
We tailored our choice of encoder architecture and pretraining technique to each dataset, mostly based on the performance of publicly available model checkpoints in evaluations using randomly acquired data.
For MNIST we used the two-convolutional-layer architecture from \citet{kirsch2019batchbald}, with dropout disabled, and SimCLR pretraining \citep{chen2020simple}; for dSprites, the architecture from \citet{burgess2017understanding} and $\beta$-TCVAE pretraining \citep{chen2018isolating}; for CIFAR-10, the ResNet-18 architecture from \citet{he2015deep} and MoCo v2 pretraining \citep{chen2020improved}; for ImageNet, the ViT-B/4 architecture from \citet{dosovitskiy2021image} and MSN pretraining \citep{assran2022masked}.
In our choice of prediction heads we sought a simple off-the-shelf option.
For MNIST, dSprites and CIFAR-10 we therefore used a random forest \citep{breiman2001random} configured using the default settings in Scikit-learn \citep{pedregosa2011scikitlearn}.
This prediction head performed poorly on ImageNet, so there we used a dropout-enabled fully connected neural network with a dropout rate of 0.1 and three hidden layers of 128 units.

\paragraph{Fully supervised models}

We chose fully supervised models to match the size of the semi-supervised models.
For MNIST we used the two-convolutional-layer model from \citet{kirsch2019batchbald} with their dropout rate of 0.5; for dSprites, the encoder from \citet{burgess2017understanding} combined with a stack of dropout-enabled fully connected layers (with a dropout rate of 0.05); for CIFAR-10, a ResNet-18 modified to incorporate two additional dropout-enabled fully connected layers (with a dropout rate of 0.5) before the final layer, as used by \citet{kirsch2021test}.

\paragraph{Alternative semi-supervised models}

In order to investigate the sensitivity of BALD and EPIG acquisition to changes in the semi-supervised model, we looked at alternative model setups for MNIST.
For the encoder we tried using the same architecture as we used for dSprites but trained as part of a standard variational autoencoder \citep{kingma2014auto,rezende2014stochastic}.
For the prediction head we tried a fully connected neural network with one hidden layer of 128 units, using Laplace approximation to infer a parameter distribution (for this we used a standard Gaussian prior, $\mathcal{N}(0, I)$, and a diagonal, tempered posterior \citep{aitchison2021statistical}, with tempering implemented by raising the likelihood term to a power of $\dim(\theta_h)$, the parameter count of the prediction head).

\subsubsection{Active learning}

We initialised the training dataset by sampling two examples from each class, then we ran the loop described in \Cref{sec:background} until the training-label budget (1,000 for ImageNet; 300 for the other datasets) was used up.
We evaluated a number of acquisition methods: random, BALD \citep{houlsby2011bayesian}, EPIG \citep{bickfordsmith2023prediction}, maximum entropy \citep{settles2008analysis}, BADGE \citep{ash2020deep}, greedy $k$ centres \citep{sener2018active}, $k$ means \citep{pourahmadi2021simple}, ProbCover \citep{yehuda2022active} and TypiClust \citep{hacohen2022active}.
To estimate BALD and EPIG we used the same setup as \citet{bickfordsmith2023prediction}: we used 100 realisations of $\theta_h$ (for random forests this meant using the individual trees; otherwise this meant sampling from the parameter distribution) and 100 samples of $x_*$ from a finite set of unlabelled inputs representative of the downstream task.
After each time the model was trained we evaluated its predictive accuracy on the test dataset.
We ran the entire active-learning process with each acquisition function 20 times with different random-number-generator seeds, then plotted the test accuracy (mean $\pm$ standard error) as a function of the size of the training dataset.

\subsubsection{Training neural-network prediction heads}
\label{sec:nn_training}

We used a loss function comprising the negative log likelihood (NLL) of the model parameters on the training data along with an $l_2$ regulariser on the parameters.
We ran up to 50,000 steps of gradient-based optimisation with a learning rate of 0.01.
To mitigate overfitting we recorded the NLL on a small validation set (roughly a fifth of the size of the training-label budget), stopped training if this NLL did not decrease for more than 5,000 consecutive steps, then restored the model parameters to the setting that gave the lowest validation-set NLL.
\section{Resources}
\label{sec:resources}

\subsection{Software}

\begin{table}[h]
    \centering
    \scriptsize
    \begin{tabular}{llll}
        \toprule
        Project           & Citation                         & License        & URL                                              \\
        \midrule
        BatchBALD Redux   & \citet{kirsch2019batchbald}      & Apache 2.0     & \shorturl{github.com/blackhc/batchbald_redux}    \\
        Bayesianize       & \citet{li2020bayesianize}        & MIT            & \shorturl{github.com/microsoft/bayesianize}      \\
        Disentangling VAE & \citet{dubois2019disentangling}  & MIT            & \shorturl{github.com/yanndubs/disentangling-vae} \\
        Faiss             & \citet{johnson2019billion}       & MIT            & \shorturl{github.com/facebookresearch/faiss}     \\
        Hydra             & \citet{yadan2019hydra}           & MIT            & \shorturl{github.com/facebookresearch/hydra}     \\
        Jupyter           & \citet{kluyver2016jupyter}       & BSD (3-clause) & \shorturl{jupyter.org}                           \\
        Laplace           & \citet{daxberger2021laplace}     & MIT            & \shorturl{github.com/aleximmer/laplace}          \\
        Lightly           & \citet{susmelj2020lightly}       & MIT            & \shorturl{github.com/lightly-ai/lightly}         \\
        Matplotlib        & \citet{hunter2007matplotlib}     & PSF (modified) & \shorturl{matplotlib.org}                        \\
        MSN               & \citet{assran2022masked}         & CC BY-NC 4.0   & \shorturl{github.com/facebookresearch/msn}       \\
        NumPy             & \citet{harris2020array}          & BSD (3-clause) & \shorturl{numpy.org}                             \\
        Pandas            & \citet{mckinney2010data}         & BSD (3-clause) & \shorturl{pandas.pydata.org}                     \\
        PyTorch           & \citet{paszke2019pytorch}        & BSD-style      & \shorturl{pytorch.org}                           \\
        Robustness        & \citet{engstrom2019robustness}   & MIT            & \shorturl{github.com/madrylab/robustness}        \\
        Scikit-learn      & \citet{pedregosa2011scikitlearn} & BSD (3-clause) & \shorturl{scikit-learn.org}                      \\
        SciPy             & \citet{virtanen2020scipy}        & BSD (3-clause) & \shorturl{scipy.org}                             \\
        solo-learn        & \citet{dacosta2022sololearn}     & MIT            & \shorturl{github.com/vturrisi/solo-learn}        \\
        tqdm              & \citet{dacostaluis2019tqdm}      & MPL            & \shorturl{github.com/tqdm/tqdm}                  \\
        \bottomrule
    \end{tabular}
    \caption{
        This work builds on a number of software projects.
    }
    \label{tab:software}
\end{table}

\subsection{Compute}

We used an internal cluster equipped with Nvidia GeForce RTX 2080 Ti (12GB) and Nvidia Titan RTX (24GB) chips.
Our cheapest run (dSprites data, semi-supervised model, random acquisition) took 15 minutes on an Nvidia GeForce RTX 2080 Ti.
Our most expensive run (CIFAR-10 data, fully supervised model, EPIG acquisition) took 200 hours on an Nvidia Titan RTX.

\end{document}